\documentclass[a4paper,12pt]{article}
\usepackage[margin=0.8in]{geometry}
\usepackage{amsmath}
\usepackage{algorithmic}
\usepackage{algorithm}
\usepackage{amsthm}
\usepackage{amsfonts}
\usepackage{comment}
\usepackage{graphicx}
\usepackage{hyperref}
\usepackage{color}

\newtheorem{theorem} {Theorem}
\newtheorem{lemma} {Lemma}
\newtheorem{definition} {Definition}
\newtheorem{corollary} {Corollary}

\newcommand{\R}{\mathcal{R}}
\newcommand{\mP}{\mathcal{P}}
\newcommand{\mK}{\mathcal{K}}
\newcommand{\mS}{\mathcal{S}}
\newcommand{\ball}{\mathbb{B}}
\newcommand{\mD}{\mathcal{D}}
\newcommand{\mV}{\mathcal{V}}
\newcommand{\mO}{\mathcal{O}}

\newcommand{\regret}{\textrm{regret}}

\def\reals{{\mathbb R}}

\newcommand{\oraclep}{\mathcal{O}_{\mP}}
\newcommand{\E}{\mathbb{E}}

\title{A Linearly Convergent Conditional Gradient Algorithm with Applications to Online and Stochastic Optimization}
\author{Dan Garber\\  
\small{Technion - Israel Inst. of Tech.} \\
\small{dangar@tx.technion.ac.il} 
\and Elad Hazan \\
\small{Technion - Israel Inst. of Tech.} \\ 
\small{ehazan@ie.technion.ac.il}}

\begin{document}

 \maketitle

\begin{abstract}
Linear optimization is many times algorithmically simpler than non-linear convex optimization. Linear optimization over matroid polytopes, matching polytopes and path polytopes are example of problems for which we have simple and efficient combinatorial algorithms, but whose non-linear convex counterpart is harder and admits significantly less efficient algorithms. This motivates the computational model of  convex optimization, including the offline, online and stochastic settings, using a linear optimization oracle.  
In this computational model we give several new results that improve over the previous state-of-the-art. 
Our main result is a novel conditional gradient algorithm for smooth and strongly convex optimization over polyhedral sets that performs only a single linear optimization step over the domain on each iteration and enjoys a linear convergence rate. This gives an exponential improvement in convergence rate over previous results. 

Based on this new conditional gradient algorithm we give the first algorithms for online convex optimization over polyhedral sets that perform only a single linear optimization step over the domain while having optimal regret guarantees, answering an open question of Kalai and Vempala, and Hazan and Kale. Our online algorithms also imply conditional gradient algorithms for non-smooth and stochastic convex optimization with the same convergence rates as projected (sub)gradient methods.
\end{abstract}

\pagestyle{myheadings}
\thispagestyle{plain}
\markboth{A Linearly Convergent CG Algorithm with Applications}{A Linearly Convergent CG Algorithm with Applications}

\section{Introduction}
First-order optimization methods, such as (sub)gradient-descent methods \cite{Shor,Nemirovski83,Nesterov07} and conditional-gradient methods \cite{FrankWolfe, Dunn80, Clarkson, Hazan08, Jaggi13b}, are often the method of choice for coping with very large scale optimization tasks. While theoretically attaining inferior convergence rate compared to other efficient optimization algorithms (e.g. interior point methods \cite{NesterovNemirovskii94}), modern optimization problems are often so large that using second-order information or other super-linear operations becomes practically infeasible. 

The computational bottleneck of (sub)gradient descent methods in many settings is the computation of orthogonal projections onto the convex domain. This is also the case with proximal methods \cite{Nesterov07}. Computing such projections is very efficient for simple domains such as the euclidean ball, the hypercube and the simplex but much more involved for more complicated domains, making these methods impractical for such problems in high-dimensional settings.

On the other hand, for many convex sets of interest, optimizing a linear objective over the domain could be done by a very efficient and simple combinatorial algorithm. Prominent examples for this phenomena are the matroid polytope for which there is a simple greedy algorithm for linear optimization, and the flow polytope (convex hull of all $s-t$ paths in a directed acyclic graph) for which linear optimization amounts to finding a minimum-weight path \cite{schrijver}.  Other important examples include the set of rotations for which linear optimization is very efficient using Wahba's algorithm \cite{wahba}, and the bounded cone of positive semidefinite matrices, for which linear optimization amounts to a leading eigenvector computation whereas projections require cimputing the \textit{singular value decomposition}. 

This phenomena motivates the study of optimization algorithms that require only linear optimization steps over the domain and their linear oracle complexity - that is, the number of linear objectives that the algorithm needs to minimize over the domain in order to achieve a desired accuracy with respect to the optimization objective.

The main contribution of this work is a conditional gradient (aka Frank-Wolfe) algorithm for oflline smooth and strongly convex optimization over polyhedral sets that requires only a single linear optimization step over the domain on each iteration and enjoys a linear convergence rate, an exponential improvement over previous results in this setting.

\begin{table}[H]
\begin{center} \label{table}
    \begin{tabular}{ | l | l | l |}
    \hline
    Setting & Previous & This paper \\ \hline
    Offline, smooth and strongly convex & $t^{-1}$ \cite{Jaggi13b} & $e^{-\Theta(t)}$ \\ \hline
    Offline, non-smooth and convex & $t^{-1/3}$ \cite{Hazan12} & $t^{-1/2}$ \\ \hline
    Offline, non-smooth and strongly convex & $t^{-1/3}$ \cite{Hazan12} & $\log(t)/t$ \\ \hline
    Stochastic, non-smooth and convex& $t^{-1/3}$ \cite{Hazan12}  & $t^{-1/2}$ \\ \hline
    Stochastic, non-smooth and strongly convex & $t^{-1/3}$ \cite{Hazan12}  & $\log{t}/t$ \\ \hline
    Online, convex losses & $T^{3/4}$ \cite{Hazan12} & $\sqrt{T}$  \\ \hline
    Online,  strongly convex losses &$T^{3/4}$ \cite{Hazan12} & $\log{T} $ \\ \hline
    \end{tabular}
\caption{Comparison between conditional gradient-based methods for optimization over polytopes in various settings. In the offline and stochastic settings we give the approximation error after $t$ linear optimization steps over the domain and $t$ gradient vector evaluations. In the online setting we give the order of the regret in a game of length $T$, and after at most $T$ 
linear optimization steps over the domain. In all results we omit the dependencies on constants and the dimension, these dependencies will be fully detailed in the sequel.}
\end{center}
\end{table}

We also consider the setting of \textit{online convex optimization} \cite{Zinkevich03, ShalevShwartz12, Hazan09, Warmuth10}. In this setting, a decision maker is iteratively required to choose a point in a fixed convex decision set. After choosing his point, an adversary chooses some convex function and the decision maker incurs a loss that is equal to the function evaluated at the point chosen. In this adversarial setting there is no hope to play as well as an optimal offline algorithm that has the benefit of hindsight. Instead the standard benchmark is an optimal naive offline algorithm that has the benefit of hindsight but must play the same fixed point on each round. The difference between the cumulative loss of the decision maker and that of of this offline benchmark is known as \textit{regret}. Based on our new linearly converging conditional gradient algorithm, we give algorithms for online convex optimization over polyhedral sets that perform only a single linear optimization step over the domain on each iteration while enjoying optimal regret guarantees in terms of the game length, answering an open question of Kalai and Vempala \cite{KalaiVempala}, and Hazan and Kale \cite{Hazan12}. Using existing techniques we give an extension of this algorithm to the {partial information} setting which obtains the best known regret bound for this setting.

Finally, our online algorithms also imply conditional gradient-like algorithms for offline non-smooth convex optimization and stochastic convex optimization that enjoys the same convergence rates as projected (sub)gradient methods in terms of the accuracy parameter $\epsilon$(albeit different dependency on constants and the dimension), but replacing the projection step of (sub)gradient methods with a single linear optimization step, again improving over the previous state of the art in these settings.

Our results are summarized in Table \ref{table}.

\subsection{Related work}
\paragraph*{The conditional gradient method for smooth optimization}
Conditional gradient methods for offline minimization of convex and smooth functions date back to the work of Frank and Wolfe \cite{FrankWolfe} which presented a method for smooth convex optimization over polyhedral sets whose iteration complexity amounts to a single linear optimization step over the convex domain. More recent works of Clarkson \cite{Clarkson}, Hazan \cite{Hazan08} and Jaggi \cite{Jaggi13b} consider the conditional gradient method for the cases of smooth convex optimization over the simplex, semidefinite cone and arbitrary convex and compact sets respectively. Despite its relatively slow convergence rate - additive error of the order $1/t$ after $t$ iterations, the benefit of the method is twofold: i) its computational simplicity - each iteration is comprised of optimizing a linear objective over the set and ii) it is known to produce sparse solutions (for the simplex this means only a few non zeros entries, for the semidefinite cone this means that the solution has low rank). Due to these two properties, conditional gradient methods have attracted much attention in the machine learning community in recent years, see \cite{Jaggi10, Jaggi13a, Jaggi13b, Dudik12a, Dudik12b, ShalevShwartz11, Laue12, Bach12}. 

It is known that in general the convergence rate $1/t$ is also optimal for this method without further assumptions, as shown in \cite{Clarkson, Hazan08, Jaggi13b}. In case the objective function is both smooth and strongly convex, there exist extensions of the basic method which achieve faster rates under various assumptions.  One such extension of the conditional-gradient algorithm with linear convergence rate was presented by Migdalas \cite{Migdalas}, however the algorithm requires to solve a regularized linear problem on each iteration which is computationally equivalent to computing projections. This is also the case with the algorithm for smooth and strongly convex optimization in the recent work of Lan \cite{Lan13}. In case the convex set is a polytope, Gu{\'{e}}Lat and Marcotte \cite{GueLat1986} has shown that the algorithm of Frank and Wolfe \cite{FrankWolfe} converges in linear rate assuming that the optimal point in the polytope is bounded away from the boundary. The convergence rate is proportional to a quadratic of the distance of the optimal point from the boundary. We note that in case the optimum lies in the interior of the convex domain, then the problem is in fact an unconstrained convex optimization problem and solvable via much more efficient methods. Gu{\'{e}}Lat and Marcotte \cite{GueLat1986} also gave an improved algorithm based on the concept of ``away steps" with a linear convergence rate that holds under weaker conditions, however this linear rate still depends on the location of the optimum with respect to the boundary of the set which may result in an arbitrarily bad convergence rate. We note that the suggestion of using ``away steps" to accelerate the convergence of the FW algorithm for strongly convex objectives was already made by Wolfe himself in \cite{Wolfe1970}. Beck and Taboule \cite{BeckTaboule} gave a linearly converging conditional gradient algorithm for solving convex linear systems, but as in \cite{GueLat1986}, their convergence rate depends on the distance of the optimum from the boundary of the set.
Here we emphasize that in this work we do not make any assumptions on the location of the optimum in the convex domain and our convergence rates are independent of it.

Ahipasaoglu, Sun and Todd \cite{Ahipasaoglu08} gave a variant of the conditional gradient algorithm with \textit{away steps} that achieves a linear convergence rate for the specific case in which the convex domain is the unit simplex. Their work also does not specify the precise dependency of the convergence rate on parameters of the problem such as the dimension, which is of great importance. In this work we derive, as an illustrating example, a linearly converging algorithm for the unit simplex. Our generalization to arbitrary polytopes is highly non-trivial and is indeed the technical heart of this work. We also provide convergence rates with detailed dependencies on natural parameters of the problem.

After our work first appeared \cite{Garber13b}, Jaggi and Lacoste-Julien \cite{Jaggi13c} presented a refined analysis of a variant of the conditional gradient algorithm with away steps from \cite{GueLat1986} that achieves a linear convergence rate without the assumption on the location of the optimum as in the original work of \cite{GueLat1986}. Their algorithm is also shown to be affine invariant. Their convergence rate however is not given explicitly and its dependency on the dimension or other natural parameters of the problem is not clear.

\paragraph*{Conditional gradient-like methods for online, stochastic and non-smooth optimization}
The two closest works to ours are those of Kalai and Vempala\cite{KalaiVempala}  and Hazan and Kale \cite{Hazan12}, both present projection-free algorithms for online convex optimization in which the only optimization carried out by the algorithms on each iteration is the minimization of a single linear objective over the decision set. \cite{KalaiVempala} gives a random algorithm for the online setting in the special case in which all loss functions are linear, also known as \textit{online linear optimization}. In this setting their algorithm achieves regret of $O(\sqrt{T})$ which is optimal \cite{CesaBianchi}. On iteration $t$ their algorithm plays a point in the decision set that minimizes the cumulative loss on all previous iterations plus a vector whose entries are independent random variables. The work of \cite{Hazan12} introduces algorithms for stochastic and online optimization which are based on ideas similar to ours - using the conditional gradient update step to approximate the steps a meta-algorithm for online convex optimization known as \textit{Regularized Follow the Leader} (RFTL) \cite{Hazan09, ShalevShwartz12}. For stochastic optimization, in case that all loss functions are smooth they achieve an optimal convergence rate of $1/\sqrt{T}$, however for non-smooth stochastic optimization they only get convergence rate of $T^{-1/3}$ and for the full adversarial setting of online convex optimization they get suboptimal regret that scales like $T^{3/4}$.

In a recent work, Lan \cite{Lan13} showed how to apply the conditional gradient algorithm to offline non-smooth optimization via a well known smoothing technique (also employed in \cite{Hazan12}). His analysis shows that an $\epsilon$ additive error is guaranteed after a total of $O(\epsilon^{-2})$ linear optimization steps over the domain and $O(\epsilon^{-4})$ calls to the subgradient oracle of the objective. Our algorithm for the non-smooth setting guarantees an $\epsilon$ additive error after $O(\epsilon^{-2})$ linear optimization steps over the domain and $O(\epsilon^{-2})$ calls to the subgradient oracle.

Also relevant to our work is the very recent work of Harchaoui, Juditsky and Nemirovski \cite{Harchaoui} who give methods for i) minimizing a norm over the intersection of a cone and the level set of a convex smooth function and ii) minimizing the sum of a convex smooth function and a multiple of a norm over a cone. Their algorithms are extensions of the conditional gradient method that assume the availability of a stronger oracle that can minimize a linear objective over the intersection of the cone and a unit ball induced by the norm of interest. They present several problems of interest for which such an oracle could be implemented very efficiently, however in general such an oracle could be computationally much less efficient than the linear oracle required by standard conditional gradient methods.

\subsection{Paper Structure} 
The rest of the paper is organized as follows. In section \ref{sec:prelim} we give preliminaries, including notation and definitions that will be used throughout this work, overview of the conditional gradient method and describe the settings of online convex optimization and stochastic optimization. In section \ref{sec:results} we give an informal statement of the results presented in this work. In section \ref{sec:offline_opt} we present our main result - a new linearly convergent conditional gradient algorithm for offline smooth and strongly convex optimization over polyhedral sets. In section \ref{sec:llo} we present and analyse our main new algorithmic machinery which we refer to as a \textit{local linear optimization oracle}. In section \ref{sec:online} we present and analyze our algorithms for online and stochastic optimization, and finally in section \ref{sec:lowerbounds} we discuss a lower bound for the problem of minimizing a smooth and strongly convex function using only linear optimization steps - showing that the oracle complexity of our new algorithm presented in section \ref{sec:offline_opt} is nearly optimal.

\section{Preliminaries}\label{sec:prelim}
We denote by $\ball_r(x)$ the euclidean ball of radius $r$ centred at $x$. We denote by $\Vert{x}\Vert$ the $\ell_2$ norm of the vector $x$ and by $\Vert{A}\Vert$ the \textit{spectral norm} of the matrix $A$, that is $\Vert{A}\Vert = \max_{x\in\ball}\Vert{Ax}\Vert$. Given a matrix $A$, we denote by $A(i)$ the vector that corresponds to the $i$th row of $A$.

\begin{definition}
We say that a function $f(x):\mathbb{R}^n\rightarrow\mathbb{R}$ is Lipschitz with parameter $L$ over the set $\mK$ if for all $x,y\in\mK$ it holds that
\begin{eqnarray*}
\vert{f(x) - f(y)}\vert \leq L\Vert{x-y}\Vert .
\end{eqnarray*}
\end{definition}

\begin{definition}
We say that a function $f(x):\mathbb{R}^n\rightarrow\mathbb{R}$ is $\beta$-smooth over the set $\mK$ if for all $x,y\in\mK$ it holds that
\begin{eqnarray*}
f(y) \leq f(x) + \nabla{}f(x)\cdot(y-x) + \frac{\beta}{2}\Vert{x-y}\Vert^2 .
\end{eqnarray*}
\end{definition}

\begin{definition}
We say that a function $f(x):\mathbb{R}^n\rightarrow\mathbb{R}$ is $\sigma$-strongly convex over the set $\mK$ if for all $x,y\in\mK$ it holds that
\begin{eqnarray*}
f(y) \geq f(x) + \nabla{}f(x)\cdot(y-x) + \frac{\sigma}{2}\Vert{x-y}\Vert^2 .
\end{eqnarray*}
\end{definition}

The above definition together with first order optimality conditions imply that for a $\sigma$-strongly convex $f$, if $x^*$ is the unique minimizer of $f$ over $\mK$, then for all $x\in\mK$ it holds that
\begin{eqnarray}\label{eq:strongconvexdist}
f(x) - f(x^*) \geq \frac{\sigma}{2}\Vert{x-x^*}\Vert^2 .
\end{eqnarray}

Note that a sufficient condition for a twice-differential function $f$ to be $\beta$-smooth and $\sigma$-strongly convex over a domain $\mathcal{K}$ is that
\begin{eqnarray*}
\forall{x\in\mathcal{K}}: \quad \beta\textrm{I} \succeq \nabla^2f(x) \succeq \sigma\textrm{I}.
\end{eqnarray*}

Let $\mP$ be a polytope described by linear equations and inequalities, i.e., $$\mP = \lbrace{x\in\mathbb{R}^n \, | \, A_1x = b_1, \, A_2x \leq b_2}\rbrace ,$$ where $A_2\in\mathbb{R}^{m\times{n}}$. We assume without loss of generality that all rows of $A_2$ are scaled to have unit $\ell_2$ norm.We denote by $\mathcal{V}(\mP)$ the set of vertices of $\mP$, and we let $N=\vert{\mathcal{V}}\vert$. 
We now define several geometric parameters of $\mP$ that will come up naturally in the analysis of our algorithms. We denote the Euclidean diameter of $\mP$ by $D(\mP)$, i.e., $D(\mP)= \max_{x,y\in\mP}\Vert{x-y}\Vert$. We denote $$\xi(P) = \min_{v\in\mathcal{V}(\mP)}\left({\min\{b_2(j)-A_2(j)\cdot{}v \, | \, j\in[m], A_2(j)\cdot{}v < b_2(j)\}}\right) .$$ That is, given an inequality constraint that defines the polytope and a vertex of the polytope, the vertex either satisfies the constraint with equality or is at least $\xi(P)$-far from satisfying it with equality. Let $r(A_2)$ denote the row-rank of the matrix $A_2$. Let $\mathbb{A}(\mP)$ denote the set of all $r(A_2)\times{n}$ matrices whose rows are linearly independent vectors chosen from the rows of $A_2$ and denote $\psi(\mP) = \max_{M\in\mathbb{A}(\mP)}\Vert{M}\Vert$. Finally denote $\mu(\mP) = \frac{\psi(\mP)D(\mP)}{\xi(\mP)}$. It is important to note that the quantity $\mu(\mP)$ is invariant to translation, rotation and scaling of the polytope $\mP$. Note also that it always holds that $\mu(\mP) \geq 1$ (this follows since by definition $\psi(\mP) \geq 1$ and $\xi(\mP) \leq D(\mP)$).
Henceforth we shall use the shorthand notation of $\mathcal{V}, D, \xi, \psi, \mu$ when the polytope considered is clear from context.
Note that in many settings of interest (problems for which there is indeed an highly-efficient algorithm for linear optimization over the specified polytope), estimating the parameters $\xi, \psi$ is straightforward. For instance, in convex domains that arise in combinatorial optimization such at the flow polytope, matching polytope, matroid polytopes, etc.

Throughout this work we will assume that we have access to an oracle that returns a vertex of $\mP$ that minimizes the dot product with a given linear objective. That is we are given a procedure $\oraclep:\mathcal{V}\rightarrow\mathbb{R}$ such that for all $c\in\mathbb{R}^n$, $\oraclep(c)\in\arg\min_{v\in\mathcal{V}}v\cdot{}c$. We call $\oraclep$ a linear optimization oracle.

\subsection{The Conditional Gradient Method and Local Linear Optimization Oracles}\label{subsec:cond_grad_intro}
The conditional gradient method is a simple algorithm for minimizing a smooth and convex function $f$ over a convex set $\mP$ - which in this work we assume to be a polytope. The appeal of the method is that it is a first order \textit{feasible point} method, i.e., the iterates always lie inside the convex set and thus no projections are needed. Further more, the update step on each iteration simply requires to minimize a linear objective over the set. The basic algorithm is given below.

\begin{algorithm}[H]
\caption{Conditional Gradient}
\label{alg:condgrad}
\begin{algorithmic}[1]
\STATE Input: sequence of step sizes $\lbrace{\alpha_t}\rbrace_{t=1}^{\infty}\subseteq[0,1]$
\STATE Let $x_1$ be an arbitrary point in $\mP$.
\FOR{$t = 1,2,...$}
\STATE $p_{t} \gets \oraclep(\nabla{}f(x_t))$
\STATE $x_{t+1} \gets x_t + \alpha_t(p_t - x_t)$ for $\alpha_t\in[0,1]$
\ENDFOR
\end{algorithmic}
\end{algorithm}

Let $x^*$ denote the unique minimizer of $f$ over $\mP$ that is, $x^* = \arg\min_{x\in\mK}f(x)$. The convergence of algorithm \ref{alg:condgrad} is due to the following simple observations.
\begin{eqnarray}\label{old_fw_anal}
&  f(x_{t+1}) - f(x^*) \\
&= f(x_t + \alpha_t(p_t - x_t)) - f(x^*) \nonumber\\
&\leq  f(x_t) - f(x^*) + \alpha_t(p_t-x_t)\cdot\nabla{}f(x_t) + \frac{\alpha_t^2\beta}{2}\Vert{p_t-x_t}\Vert^2 & \quad  \textrm{/ $\beta$-smoothness of $f$} \nonumber \\
&\leq  f(x_t) - f(x^*) + \alpha_t(x^*-x_t)\cdot\nabla{}f(x_t) + \frac{\alpha_t^2\beta}{2}\Vert{p_t-x_t}\Vert^2 & \quad \textrm{/ optimality of $p_t$} \nonumber \\
&\leq  f(x_t) - f(x^*) + \alpha_t(f(x^*)-f(x_t)) + \frac{\alpha_t^2\beta}{2}\Vert{p_t-x_t}\Vert^2 & \quad \textrm{/ convexity of $f$} . \nonumber %\\
%&\leq  (1-\alpha_t)(f(x_t)-f(x^*)) + \alpha_t^2\beta{}D^2 \nonumber
\end{eqnarray}
Thus for an appropriate choice for the sequence of step sizes $\lbrace{\alpha_t}\rbrace_{t=1}^{\infty}$, the approximation error strictly decreases on each iteration. This leads to the following theorem (for a proof see for instance the modern survey of \cite{Jaggi13b}).
\begin{theorem}
There is an explicit choice for the sequence of step sizes $\lbrace{\alpha_t}\rbrace_{t=1}^{\infty}$ such that for every $t\geq 2$, the iterate $x_t$ of Algorithm \ref{alg:condgrad} satisfies that $f(x_t) - f(x^*) = O\left({\frac{\beta{}D^2}{t-1}}\right)$.
\end{theorem}

The relatively slow convergence of the conditional gradient algorithm is due to the term $\Vert{p_t-x_t}\Vert$ in Eq. \eqref{old_fw_anal}, that may remain as large as the diameter of $\mP$ while the term $f(x_t)-f(x^*)$ keeps on shrinking, that forces choosing values of $\alpha_t$ that decrease like $\frac{1}{t}$ in order to guarantee convergence \cite{Clarkson, Hazan08, Jaggi13b}.

Notice that if $f(x)$ is $\sigma$-strongly convex for some $\sigma>0$ then according to Eq. \eqref{eq:strongconvexdist},
knowing that for some iteration $t$ it holds that $f(x_t)-f(x^*) \leq \epsilon$, implies that $\Vert{x_t-x^*}\Vert^2 \leq \frac{2\epsilon}{\sigma}$. Thus when choosing the point $p_t$, denoting $r=\sqrt{2\epsilon/\sigma}$, it is enough to consider points that lie in the intersection set  $\mP\cap\ball_r(x_t)$, i.e., take $p_t$ to be the solution to the optimization problem
\begin{eqnarray}\label{eq:locallinearopt}
\min_{p\in\mP\cap\ball_r(x_t)}p\cdot\nabla{}f(x_t) .
\end{eqnarray}

In this case the term $\Vert{p_t-x_t}\Vert^2$ in Eq. \eqref{old_fw_anal} will be of the same magnitude as $f(x_t)-f(x^*)$ (or even smaller) and as observable in Eq. \eqref{old_fw_anal}, a linear convergence rate will follow.

However, solving Problem \eqref{eq:locallinearopt} is potentially much more difficult than solving the original linear problem $\min_{p\in\mP}p\cdot\nabla{}f(x_t)$, and is not straight-forward solvable using the linear optimization oracle of $\mP$. 

To overcome the problem of solving the linear problem in the intersection $\mP\cap\ball_r(x_t)$ we introduce the following definition which is a primary ingredient of our work.

\begin{definition}[Local Linear Optimization Oracle]
We say that a procedure $\mathcal{A}(x,r,c)$, where $x\in\mP$, $r\in\mathbb{R}^+$, $c\in\mathbb{R}^n$, is a Local Linear Optimization Oracle with parameter $\rho \geq 1$ for the polytope $\mP$, if $\mathcal{A}(x,r,c)$ returns a feasible point $p\in\mP$ such that:
\begin{enumerate}
\item $\forall{}y\in\ball(x,r)\cap\mP$ it holds that $y \cdot c \geq p \cdot c$.
\item $\Vert{x-p}\Vert \leq \rho\cdot{}r$.
\end{enumerate}
\end{definition}

The local linear optimization oracle (LLOO) relaxes Problem \eqref{eq:locallinearopt} by solving the linear problem on a larger set, but one that still has a diameter that is not much larger than $\sqrt{f(x_t)-f(x^*)}$. Our main contribution is in showing that for a polytope $\mP$, a LLOO can be constructed such that the parameter $\rho$ depends only on the dimension $n$ and the quantity $\mu(\mP)$. Moreover, the algorithmic construction requires only a single call to the original linear optimization oracle $\oraclep$. Hence, the complexity per iteration, in terms of the number of calls to the linear optimization oracle $\oraclep$, remains the same as the original conditional gradient algorithm (Algorithm \ref{alg:condgrad}).

\subsection{Online convex optimization and its application to stochastic and offline optimization}
The problem of online convex optimization (OCO) \cite{Zinkevich03, Hazan07, Hazan09} takes the form of the following repeated game. A decision maker is required on each iteration $t$ of the game to choose a point $x_t\in\mK$, where $\mK$ is a fixed convex set. After choosing the point $x_t$, a convex loss function $f_t(x)$ is reveled, and the decision maker incurs loss $f_t(x_t)$. The emphasis in this model is that the loss function on time $t$ may be chosen completely arbitrarily and even in an adversarial manner given the current and past decisions of the decision maker. In the \textit{full information} setting, after making his decision on time $t$, the decision maker gets full knowledge of the function $f_t$. In the \textit{partial information} setting (\textit{bandit}) the decision maker only learns the value $f_t(x_t)$ and does not gain any other knowledge about $f_t$.

The standard goal in this setting is to have overall loss which is not much larger than that of the best fixed point in $\mK$, in hindsight. Formally the goal is to minimize a quantity known has \textit{regret} which is given by
\begin{eqnarray*}
\textrm{regret}_T := \sum_{t=1}^Tf_t(x_t) - \min_{x\in\mK}\sum_{t=1}^Tf_t(x) .
\end{eqnarray*}

In certain cases, such as in the bandit setting, the decision maker must use randomness in order to make his decisions. In this case we consider the expected regret, where the expectation is taken over the randomness in the algorithm of the decision maker.

In the \textit{full information setting} and for general convex losses the optimal regret bound attainable scales like $\sqrt{T}$ \cite{CesaBianchi} where $T$ is the length of the game. In the case that all loss functions are strongly convex, the optimal regret bound attainable scales like $\log(T)$ \cite{Hazan11}.

\subsubsection{Algorithms for OCO}\label{sec:RFTL}

A simple algorithm that attains optimal regret of $O(\sqrt{T})$ for general convex losses is known as the Regularized Follows The Leader algorithm (RFTL) \cite{Hazan09}. On time $t$ the algorithm predicts according to the following rule.
\begin{eqnarray}\label{rftl_update}
x_t \gets \arg\min_{x\in\mK} \left\{ \eta\sum_{\tau=1}^{t-1}\nabla{}f_{\tau}(x_{\tau})\cdot x + \R(x) \right\} .
\end{eqnarray}
Where $\eta$ is a parameter known as the learning rate and $\R$ is a strongly convex function known as the regularization.
From an offline optimization point of view, achieving low regret is thus equivalent to minimizing a single strongly-convex objective over the feasible set per iteration.  In fact, with the popular choice $\R(x) = \Vert{x}\Vert^2$, we get that Problem \eqref{rftl_update} is just the minimization of a function that is both smooth and strongly-convex over the feasible domain $\mK$, and is in fact equivalent to computing an Euclidean projection onto $\mK$.

In case of strongly-convex losses a slight variant of Eq. \eqref{rftl_update}, which also takes the form of minimizing a smooth and strongly convex function when choosing $\R(x)=\Vert{x}\Vert^2$,  guarantees optimal $O(\log(T))$ regret.

In the \textit{partial information setting} the RFTL rule \eqref{rftl_update} with the algorithmic conversion of the \textit{bandit} problem to that of the \textit{full information} problem established in \cite{FKM05}, yields an algorithm with regret $O(T^{3/4})$, which is the best to date.

Our algorithms for online optimization are based on iteratively approximating the RFTL objective in Eq. \eqref{rftl_update} using our new linearly convergent CG algorithm for smooth and strongly convex optimization, thus replacing the projection step in \eqref{rftl_update} (in case $\R(x)=\Vert{x}\Vert^2$) with a single linear optimization step over the domain. 

We note that while the update rule in Eq. \eqref{rftl_update} uses the gradients of the loss functions which are denoted by $\nabla{}f_{\tau}$, it is in fact not required to assume that the loss functions are differentiable everywhere in the domain. It suffices to assume that the loss functions only have a sub-gradient everywhere in the domain, making the algorithm suitable also for non-smooth settings. Throughout this work we do not  differentiate between these two cases and the notation $\nabla{}f(x)$ should be understood as a gradient of $f$ at the point $x$ in case $f$ is differentiable and as a sub-gradient of $f$ in case $f$ only has a sub-gradient in this point.

\subsubsection{Stochastic optimization}\label{sec:stoc_opt}
In stochastic optimization the goal is to minimize a convex function $F(x)$ given by
\begin{eqnarray*}
F(x) = \E_{f\sim\mD}[f(x)],
\end{eqnarray*}
where $\mD$ is a fixed, yet unknown distribution over convex functions.
In this setting we don't have direct access to the function $F$, instead we assume to have a stochastic oracle for $F$ that when queried, returns a function $f$ sampled from $\mD$, independently of previous samples. 

The general setting of online convex optimization is harder than stochastic optimization in the sense that an algorithm for OCO could be directly applied to stochastic optimization as follows.  We simulate an online game of $T$ rounds for the OCO algorithm, where
on each iteration $t$ the loss function $f_t(x)$ is generated by a query to the stochastic oracle of $\mD$. 
Let us denote by $\regret_T$ an upper bound on the regret of the online algorithm with respect to \textit{any} sample of $T$ functions from the distribution $\mD$. Thus, given such a sample - $\{f_t\}_{t=1}^T$, it holds that
\begin{eqnarray*}
\sum_{t=1}^Tf_t(x_t) - \min_{x\in\mK}\sum_{t=1}^Tf_t(x) \leq \regret_T.
\end{eqnarray*}

Denoting $x^* \in \arg\min_{x\in\mK}F(x)$ we thus in particular have that
\begin{eqnarray*}
\sum_{t=1}^Tf_t(x_t) - \sum_{t=1}^Tf_t(x^*) \leq \regret_T .
\end{eqnarray*} 

Since for all $t\in[T]$ it holds that $\E[f_t(x_t) | x_t] = F(x_t)$ and $\E[f_t(x^*)]=F(x^*)$ (where in both cases the expectation is with respect to the random choice of $f_t$), taking expectation over the randomness of the oracle for $F$ we have that
\begin{eqnarray*}
\E\left[{\sum_{t=1}^Tf_t(x_t) - \sum_{t=1}^Tf_t(x^*)}\right]
 &=&\sum_{t=1}^T\E\left[{\E[f_t(x_t) | x_t ]}\right] - T\cdot{}F(x^*)  
 \\
 &=&\sum_{t=1}^T\E[F(x_t)] - T\cdot{}F(x^*)\\
 &=& \E\left[{\sum_{t=1}^TF(x_t)}\right] - T\cdot{}F(x^*) .
\end{eqnarray*} 

Denoting $\bar{x} = \frac{1}{T}\sum_{t=1}^Tx_t$ we have by convexity of $F$ that
\begin{eqnarray*}
\E[F(\bar{x})] - F(x^*) \leq \frac{\regret_T}{T} .
\end{eqnarray*} 

Thus the same regret rates that are attainable for online convex optimization hold as convergence rates, or sample complexity, for stochastic convex optimization. 
We note that using standard concentration results for martingales, one can also derive error bounds that hold with high probability and not only in expectation, but these are beyond the scope of this paper. We refer the interested reader to \cite{CesaBianchi04} for more details. 

\subsubsection{Non-smooth optimization}\label{sec:nonsmooth_opt}
As in stochastic optimization (see previous subsection), an algorithm for OCO also implies an algorithm for offline convex optimization. Thus a conditional gradient-like algorithm for OCO implies a conditional gradient-like algorithm for non-smooth convex optimization. This is in contrast to the original conditional gradient method which is suitable for smooth optimization only.

Applying an OCO algorithm to the minimization of a, potentially non-smooth, convex function $f(x)$ over a feasible convex set $\mK$, is as follows.  As in the previous subsection, we simulate a game of length $T$ for the OCO algorithm in which the loss function $f_t$ on each round is just the function to minimize $f(x)$. As in the stochastic case, denoting $\bar{x} = \frac{1}{T}\sum_{t=1}^Tx_t$, i.e., the average of iterates returned by the online algorithm, we have that
\begin{eqnarray*}
f(\bar{x}) -f(x^*) \leq \frac{1}{T}\sum_{t=1}^Tf(x_t) - f(x^*) = \frac{1}{T}\left({\sum_{t=1}^Tf_t(x_t) - f_t(x^*)}\right) = \frac{\regret_T}{T},
\end{eqnarray*}
where the first inequality follows from convexity of $f$. Hence the regret bound immediately translates to a convergence rate for offline optimization problem.

\section{Our Results}\label{sec:results}
In this section we give an informal presentation of the results presented in this paper. In all of the following results we consider optimization (either offline or online) over a polytope, denoted $\mP$, and we assume the availability an oracle $\oraclep$ that given a linear objective $c\in\mathbb{R}^n$ returns a vertex of $\mP$, $v\in\mathcal{V}$ that minimizes the dot product with $c$ over $\mP$.

\paragraph*{Offline smooth and strongly convex optimization}
Given a $\beta$-smooth, $\sigma$-strongly convex function $f(x)$ we present an iterative algorithm that after $t$ iterations returns a point $x_{t+1}\in\mP$ such that 
\begin{eqnarray}
f(x_{t+1})-f(x^*) \leq C\exp\left({-\frac{\sigma}{4\beta{}n\mu^2}t}\right),
\end{eqnarray}
where $x^*=\arg\min_{x\in\mP}f(x)$ and $C$ satisfies that $C \geq f(x_1) - f(x^*)$. 
Each iteration is comprised of a single call to the linear optimization oracle of $\mP$ and a single evaluation of a gradient vector of $f$. 

As we show in section \ref{sec:lowerbounds}, the above convergence rate is nearly tight in certain settings for a conditional gradient-like method.

\paragraph*{Online convex optimization}
We present algorithms for OCO that require only a single call to the linear optimization oracle  of $\mP$ per iteration of the game. In the following we let $G$ denote an upper bound on the $\ell_2$ norm of the (sub)gradients of the loss functions revealed throughout the game. Our results for the online setting are as follows:
\begin{enumerate}
\item An algorithm for OCO with arbitrary convex loss functions whose sequence of predictions - $\lbrace{x_t}\rbrace_{t=1}^T$ satisfies that
%regret after $T$ rounds is $O\left({GD\sqrt{n}\mu\sqrt{T}}\right)$.
\begin{eqnarray}
\sum_{t=1}^T f_t(x_t) - \min_{x \in \mP} \sum_{t=1}^T f_t(x) = O\left({GD\mu\sqrt{nT}}\right) .
\end{eqnarray}

This bound is optimal in terms of $T$ \cite{CesaBianchi}.

\item An algorithm for OCO with $\sigma$-strongly convex loss functions %whose regret after $T$ rounds is $O\left(\sigma{}D^2\rho^4 + {\frac{(G+\sigma{}D)^2n\mu^2}{\sigma}\log(T)}\right)$.
whose sequence of predictions - $\lbrace{x_t}\rbrace_{t=1}^T$ satisfies that
\begin{eqnarray}
\sum_{t=1}^T f_t(x_t) - \min_{x\in \mP} \sum_{t=1}^T f_t(x) = O\left(\sigma{}D^2\rho^4 + {\frac{(G+\sigma{}D)^2n\mu^2}{\sigma}\log(T)}\right) .
\end{eqnarray}

This bound is also optimal in terms of $T$ \cite{Hazan11}.

\item A randomized algorithm for the \textit{partial information} setting whose sequence of predictions - $\lbrace{x_t}\rbrace_{t=1}^T$ satisfies that
%expected regret after $T$ rounds is  $O\left({GD\sqrt{\frac{nD}{r_0}}T^{3/4}  + GD\sqrt{n}\mu\sqrt{T}}\right)$ .
\begin{eqnarray}
\mathbb{E}\left[{\sum_{t=1}^Tf_t(x_t) - \min_{x\in\mP}\sum_{t=1}^Tf_t(x)}\right] = O\left({GD\sqrt{\frac{nD}{r_0}}T^{3/4}  + GD\mu\sqrt{nT}}\right) .
\end{eqnarray}

Here we assume for simplicity that $\mP$ is full-dimensional and we denote by $r_0$ the size of the largest Euclidean ball enclosed in it.
%Here we assume w.l.o.g. that $\mP$ contains the origin and that $r\ball\subseteq\mP\subseteq{}R\ball$. We assume further that $\vert{f_t(x)}\vert \leq C$ for all $x\in\mP$. 
This bound matches the current state-of-the-art in this setting in terms of $T$ \cite{FKM05}.
\end{enumerate} 

\paragraph*{Stochastic and non-smooth optimization}
Applying our online algorithms to the stochastic setting, as specified in Subsection \ref{sec:stoc_opt}, yields algorithms that given a stochastic oracle for a function of the form - $F(x) = \mathbb{E}_{f\sim\mathcal{D}}[f(x)]$, and after viewing $T$ i.i.d samples from the distribution $\mathcal{D}$ and making $T$ calls to the linear optimization oracle $\oraclep$, return a point $\bar{x}=\frac{1}{T}\sum_{t=1}^Tx_t$ such that the following guarantees hold:
\begin{enumerate}
\item If $\mathcal{D}$ is a distribution over arbitrary convex functions then 
\begin{eqnarray}
\E[F(\bar{x})] - \min_{x\in\mP}F(x) = O\left({\frac{GD\sqrt{n}\mu}{\sqrt{T}}}\right) .
\end{eqnarray}

\item If $\mathcal{D}$ is a distribution over $\sigma$-strongly convex functions then 
\begin{eqnarray}
\E[F(\bar{x})] - \min_{x\in\mP}F(x) = O\left({\frac{\sigma^2D^2\rho^4 + (G+\sigma{}D)^2n\mu^2\log(T)}{\sigma{}T}}\right) .
\end{eqnarray}
 
\end{enumerate}

Here again $G$ denotes an upper bound on the $\ell_2$ norm of the (sub)gradients of the functions $f$ sampled from the distribution $\mathcal{D}$.

As described in Subsection \ref{sec:nonsmooth_opt}, the above rates (without the expectation) hold also for non-smooth convex and strongly convex optimization. 

\section{A Linearly Convergent Conditional Gradient Algorithm for Smooth and Strongly Convex Optimization over Polyhedral Sets}\label{sec:offline_opt}
In this section we consider the following offline optimization problem.
\begin{eqnarray}\label{prob:offline}
\min_{x\in\mP}f(x) ,
\end{eqnarray}
where we assume that $f$ is $\beta$-smooth and $\sigma$-strongly convex, and $\mP$ is a polytope. We further assume that we have a LLOO oracle for $\mP$ - $\mathcal{A}(x,r,c)$, as defined in Subsection \ref{subsec:cond_grad_intro} . In section \ref{sec:llo} we show that given an oracle for linear minimization over $\mP$ , such a LLOO oracle could be efficiently constructed.

The algorithm is given below.

\begin{algorithm}[H]
\caption{LLOO-based Convex Optimization}
\label{alg:offline}
\begin{algorithmic}[1]
\STATE Input: $\mathcal{A}(x,r,c)$ - LLOO with parameter $\rho \geq 1$ for polytope $\mP$
\STATE Let $x_1$ be an arbitrary vertex of $\mP$ and let $C \geq f(x_1) - f(x^*)$
\STATE $\alpha \gets \frac{\sigma{}}{2\beta{}\rho^2}$
\FOR{$t = 1,2,...$}
\STATE $r_t \gets \sqrt{\frac{2C}{\sigma}e^{-\frac{\alpha}{2}(t-1)}}$
\STATE $p_{t} \gets \mathcal{A}(x_t,r_t, \nabla{}f(x_t))$
\STATE $x_{t+1} \gets x_t + \alpha(p_t - x_t)$
\ENDFOR
\end{algorithmic}
\end{algorithm}

\begin{theorem}\label{thm:offthm}
Algorithm \ref{alg:offline}, instanciated with the LLOO implementation given in Algorithm \ref{alg:llo} (for which $\rho = \sqrt{n}\mu$, see Section \ref{sec:llo}), satisfies that
for each $t\geq 1$, the iterate $x_{t+1}$ is feasible ($x_{t+1}\in\mP$) and $$f(x_{t+1})-f(x^*) \leq C\exp\left({-\frac{\sigma}{4\beta{}n\mu^2}t}\right) ,$$ where $x^*=\arg\min_{x\in\mP}f(x)$.  Furthermore, after $t$ iterations the algorithm has made a total of $t$ calls to the linear optimization oracle of $\mP$ and $t$ gradient vector evaluations of $f(x)$. 
\end{theorem}

The theorem is a consequence of the following Lemma \ref{lem:offconv} and Lemma \ref{lem:oracle} (see Section \ref{sec:llo}). Lemma \ref{lem:offconv} proves the convergence rate of the algorithm given a black-box access to a LLOO with some arbitrary parameter $\rho$. Lemma \ref{lem:oracle} then gives an explicit construction of a LLOO with parameter $\rho = \sqrt{n}\mu$ that requires only a single call to the linear optimization oracle  per invocation.

We now turn to analyze the convergence rate of Algorithm \ref{alg:offline}. The following lemma is of general interest and will be also used in the section on online optimization.

\begin{lemma}\label{lem:fw}
Assume that $f(x)$ is $\beta$-smooth and let $x^*\in\arg\min_{x\in\mP}f(x)$. Assume that on iteration $t$ it holds that $\Vert{x_t-x^*}\Vert \leq r_t$, and let $x_{t+1} \gets x_t + \alpha(p_t-x_t)$, where $p_t$ is the output of a LLOO with parameter $\rho$ with respect to the input $(x_t,r_t,\nabla{}f(x_t))$, and let $\alpha\in[0,1]$. Then it holds that
\begin{eqnarray*}
f(x_{t+1}) - f(x^*) \leq (1-\alpha)\left({f(x_t) - f(x^*)}\right) + \frac{\beta}{2}\alpha^2\min\lbrace{\rho{}^2r_t^2, D^2}\rbrace .
\end{eqnarray*}
\end{lemma}

\begin{proof}
By the $\beta$-smoothness of $f(x)$ and the definition of $x_{t+1}$ we have that
\begin{eqnarray*}
f(x_{t+1}) &=& f(x_t + \alpha(p_t - x_t)) \leq f(x_t) + \alpha(p_t - x_t)\cdot\nabla{}f(x_t) + \frac{\beta}{2}\alpha^2\Vert{p_t - x_t}\Vert^2  .
\end{eqnarray*}
Since $\Vert{x_t - x^*}\Vert \leq r_t$, by the definition of the oracle $\mathcal{A}$ it holds that i) $p_t\cdot\nabla{}f(x_t) \leq {x^*}\cdot\nabla{}f_t(x_t)$ and ii) $\Vert{x_t - p_t}\Vert \leq \min\lbrace{\rho{}r_t, D}\rbrace$. Thus we have that
\begin{eqnarray*}
f(x_{t+1}) & \leq & f(x_t) + \alpha(x^* - x_t)\cdot\nabla{}f(x_t) + \frac{\beta}{2}\alpha^2\min\lbrace{\rho{}^2r_t^2, D^2}\rbrace .
\end{eqnarray*}
Using the convexity of $f(x)$ and subtracting $f(x^*)$ from both sides we have,
\begin{eqnarray*}
f(x_{t+1}) - f(x^*) \leq (1-\alpha)\left({f(x_t) - f(x^*)}\right) + \frac{\beta}{2}\alpha^2\min\lbrace{\rho{}^2r_t^2, D^2}\rbrace .
\end{eqnarray*}
\end{proof}

\begin{lemma}\label{lem:offconv}[Convergence of Algorithm \ref{alg:offline}]
Denote $h_t = f(x^*) - f(x_t)$. Then for all $t\geq 1$ it holds that
\begin{eqnarray*}
h_t \leq Ce^{-\frac{\sigma}{4\beta\rho^2}(t-1)} .
\end{eqnarray*}
\end{lemma}

\begin{proof}
The proof is by a simple induction. For $t=1$ we have by definition that $h_1 = f(x^*)-f(x_1) \leq C$.

Now assume that the lemma holds for $t\geq 1$. This implies via the the strong convexity of $f(x)$ (see Eq. \ref{eq:strongconvexdist}) that
\begin{eqnarray*}
\Vert{x_t-x^*}\Vert^2 \leq \frac{2}{\sigma}h_t \leq \frac{2C}{\sigma}e^{-\frac{\sigma}{4\beta\rho^2}(t-1)},
\end{eqnarray*}
where the second inequality follows from the induction hypothesis.

Thus, for $r_t = \sqrt{\frac{2C}{\sigma}e^{-\frac{\sigma}{4\beta\rho^2}(t-1)}}$ we have that
$x^* \in \mP\cap{}\ball_{r_t}(x_t)$.
Applying Lemma \ref{lem:fw} with respect to $x_t, r_t$ and using the induction hypothesis we have that
\begin{eqnarray*}
h_{t+1} & \leq & (1-\alpha)h_t + \frac{\beta}{2}\alpha^2\min\lbrace{\rho{}^2r_t^2, D^2}\rbrace\\
& \leq & (1-\alpha)Ce^{-\frac{\sigma}{4\beta\rho^2}(t-1)} + \frac{\alpha^2\beta\rho^2}{
\sigma} Ce^{-\frac{\sigma}{4\beta\rho^2}(t-1)}\\
& =& Ce^{-\frac{\sigma}{4\beta\rho^2}(t-1)}(1 - \alpha + \frac{\alpha^2\beta\rho^2}{
\sigma}) .
\end{eqnarray*}
By plugging the value of $\alpha$ from Algorithm \ref{alg:offline} and using $(1-x) \leq e^{-x}$ we have that
\begin{eqnarray*}
h_{t+1} \leq Ce^{-\frac{\sigma}{4\beta\rho^2}t} .
\end{eqnarray*}
\end{proof}

\section{Construction of a Local Linear Optimization Oracle}\label{sec:llo}

In this section we present an efficient construction of a Local Linear Optimization Oracle for a polytope $\mP$, given only an oracle for minimizing a linear objective over $\mP$.  

As an exposition for our construction for arbitrary polytopes, we first consider the specific case of constructing a LLOO for the probabilistic simplex in $\reals^n$, that is the set $\mS_n = \lbrace{x\in\mathbb{R}^n \, |\, \forall{i\in[n]}:\, x_i \geq 0 \, , \, \sum_{i=1}^nx_i =1}\rbrace$. Then we show how to generalize the simplex case to an arbitrary polytope.

\subsection{Construction of a Local Linear Optimization Oracle for the Probabalistic Simplex}

The following lemma shows that in the case of the probabilistic simplex, an LLOO could be implemented by minimizing a linear objective over the intersection of the simplex and an $\ell_1$ ball. We then show that this problem could be solved optimally by minimizing a single linear objective over the simplex (without the additional $\ell_1$ constraint).

\begin{lemma}\label{lem:simplex_lloo}
Given a point $x\in\mS_n$, a radius $r>0$ and a linear objective $c\in\mathbb{R}^n$, consider the optimization problem
\begin{eqnarray}\label{eq:llo_simplex}
&&\min_{y\in\mS_n}y \cdot c  \nonumber \\
\textrm{s.t. } &&\Vert{x-y}\Vert_1 \leq d ,
\end{eqnarray}
for some $d > 0$.
Let us denote by $p^*$ an optimal solution to Problem \eqref{eq:llo_simplex} when we set $d=\sqrt{n}r$. Then $p^*$ is the output of a LLOO with parameter $\rho =\sqrt{n}$ for $\mS_n$. That is,
\begin{enumerate}
\item $\forall{}y\in\mS_n\cap\ball_r(x):\, p^* \cdot c \leq y \cdot c$.
\item $\Vert{x-p^*}\Vert \leq \sqrt{n}r$.
\end{enumerate}
\end{lemma}

\begin{proof}
The proof follows since for any $x,y\in\reals^n$ it holds that $\frac{1}{\sqrt{n}}\Vert{x-y}\Vert_1 \leq \Vert{x-y}\Vert \leq \Vert{x-y}\Vert_1$. 
\end{proof}

Problem \eqref{eq:llo_simplex} with parameter $d=\sqrt{n}r$ is solved optimally by the following simple algorithm. 
\begin{algorithm}[H]
\caption{Local Linear Optimization Oracle for the Simplex}
\label{alg:llosimplex}
\begin{algorithmic}[1]
\STATE Input: point $x\in\mS_n$, radius $r > 0$, linear objective $c\in\mathbb{R}^n$
\STATE $d \gets \sqrt{n}r$
\STATE $\Delta \gets \min\{d/2, 1\}$
\STATE $i^* \gets \arg\min_{i\in[n]}c(i)$
\STATE $p_+ \gets \Delta\cdot e_{i^*}$ 
\STATE $p_- \gets \vec{0}$
\STATE Let $i_1,...,i_n$ be a permutation over $[n]$ such that $c(i_1) \geq c(i_2) \geq ... \geq c(i_n)$
\STATE Let $k\in[n]$ be the smallest integer such that $\sum_{j=1}^kx(i_j) \geq \Delta$ 
\STATE $\forall{j\in[k-1]}: $ $p_-(i_j) \gets x(i_j)$
\STATE $p_-(i_k) \gets \Delta - \sum_{j=1}^{k-1}x(i_j)$
\RETURN $p \gets x + p_+ - p_-$
\end{algorithmic}
\end{algorithm}

The algorithm basically modifies the input point $x$ by moving the largest amount of mass which will not violate the constraint $\Vert{x-p}\Vert_1 \leq d$ from the entries that correspond to the largest (signed) entries in the objective $c$ to the single entry that corresponds to the smallest (signed) entry in the objective $c$.

In Algorithm \ref{alg:llosimplex}, we fix the value of $d$ to $\sqrt{n}r$ to correspond to Lemma \ref{lem:simplex_lloo}. However, as the following lemma shows, the algorithm finds an optimal solution to Problem \eqref{eq:llo_simplex} for any $d \geq 0$.

\begin{lemma}\label{lem:simplex_opt}
Fix $d \geq 0$. Algorithm \ref{alg:llosimplex} finds an optimal solution to Problem \eqref{eq:llo_simplex} with parameter $d$.
\end{lemma}
\begin{proof}
Fix an optimal solution $p^*$ to Problem \eqref{eq:llo_simplex}. We can write $p^*$ in the following way:
\begin{eqnarray}\label{eq1}
p^* = x - p_{-} + p_{+}, 
\end{eqnarray}
where $p_{-},p_{+}$ are non-negative. Note that without loss of generality we can assume that $p_-,p_+$ are orthogonal. To see this, assume that there exists an entry $i$ such that $p_-(i)>0$ and $p_+(i) > 0$. By replacing $p_-,p_+$ with $p_- - \min\{p_-(i),p_+(i)\}e_i$ and $p_+ - \min\{p_-(i),p_+(i)\}e_i$ respectively, where $e_i$ is the $i$th standard basis vector in $\reals^n$, we have that the new vectors still satisfy Eq. \eqref{eq1}, both are non-negative but now at least one of them has a value of $0$ in the $i$th entry. By repeating this process for every entry $i$ that is non-zero in both vectors we can make them orthogonal. As a result, it must hold that $x \geq p_-$ (otherwise $p^*$ is not be feasible).
Furthermore, since $p^*$ is feasible ($\Vert{p^*}\Vert_1 = 1$), it must hold that $\Vert{p_+}\Vert_1 = \Vert{p_-}\Vert_1$.

Denote $\Delta =  \min\{d/2,1\}$ and assume now that $\Vert{p_+}\Vert_1 < \Delta$ (i.e. the $\ell_1$ constraint in Problem \eqref{eq:llo_simplex} is not tight for $p^*$), and denote $w = x - p_-$. It follows that there must exist a vector $y \geq 0$ such that $\Vert{y}\Vert_1 = \Delta - \Vert{p_+}\Vert_1$ and $w \geq y$.
Now define
\begin{eqnarray*}
\tilde{p}_- := p_- + y, \qquad \tilde{p}_+ := p_+ + y .
\end{eqnarray*}
Note that it holds that $p^* = x - \tilde{p}_- + \tilde{p}_+$, $\Vert{\tilde{p}_+}\Vert_1 = \Vert{\tilde{p}_-}\Vert_1= \Delta$ and that $x \geq \tilde{p}_-$ (although $\tilde{p}_-, \tilde{p}_+$ are no longer orthogonal).
%We now claim that w.l.o.g we can assume that $\Vert{x-p^*}\Vert = \min\{d/2,1\}$ (i.e. the $\ell_1$ constraint is tight).  To see this assume the contrary. It follows that there exists an entry $i\in[n]$ such that the 
%By the optimality of $p^*$ we now have that
%\begin{eqnarray*}
%0 \leq (x-p^*)\cdot c = (p_- - p_+)\cdot c = \Vert{p_+}\Vert_1\left({\frac{p_-}{\Vert{p_+}\Vert_1} - \frac{p_+}{\Vert{p_+}\Vert_1}}\right)\cdot c .
%\end{eqnarray*}
%Note that the RHS of the above equation is monotone non-decreasing in $\Vert{p_+}\Vert_1$. Thus, since by the orthogonality of $p_+,p_-$ we have that $\Vert{x-p^*}\Vert_1 =   2\Vert{p_+}\Vert_1$, without loosing generality, we can assume that $\Vert{p_+}\Vert_1 = \Vert{p_-}\Vert_1 = \min\{d/2,1\}$.

%Thus, denoting $\Delta =  \min\{d/2,1\}$, the above discussion leads us to the conclusion that w.l.o.g.,

Thus we have that
\begin{eqnarray*}
\tilde{p}_+ \in \Delta\cdot\mS_n ,\qquad \tilde{p}_- \in \left({\Delta\cdot\mS_n}\right)\cap\{z\in\reals^n \, | \, z \leq x\} .
\end{eqnarray*}

Note also that for any $p_1\in\Delta\cdot\mS_n$ and $p_2\in\left({\Delta\cdot\mS_n}\right)\cap\{z\in\reals^n \, | \, z \leq x\}$ it holds that $x+p_1-p_2$ is a feasible solution to Problem \eqref{eq:llo_simplex}.

Now we can write
\begin{eqnarray}\label{eq2} 
(p^*-x) \cdot c &=& \tilde{p}_+\cdot c - \tilde{p}_-\cdot c \geq  \min_{p_1\in\Delta\cdot \mS_n}p_1\cdot{}c - \max_{p_2\in\Delta \cdot \mS_n \, : \, p_2\leq x}p_2\cdot{}c .
\end{eqnarray}

It is now a simple observation that the vectors $p_+,p_-$ computed in Algorithm \ref{alg:llosimplex} are exactly solutions to the optimization problems 
$$\min_{p_1\in\Delta\cdot \mS_n}p_1\cdot{}c , \qquad \max_{p_2\in\Delta \cdot \mS_n \, : \, p_2\leq x}p_2\cdot{}c$$ respectively. Hence the lemma follows.

\end{proof}

Two important observations regarding the implementation of Algorithm \ref{alg:llosimplex} are that i) the running time of the algorithm does not explicitly depends on the dimension $n$ but rather on the number of non-zero entries in $x$ and the time to compute the index $i^*$ and ii) computing the index $i^*$ is equivalent to finding a vertex of $\mS_n$ that minimizes the dot product with the objective $c$, and hence is equivalent to a single call to the linear optimization oracle of $\mS_n$. 

\subsection{Construction of a Local Linear Optimization Oracle for an Arbitrary Polytope}

We now turn to generalize the above simple construction for the simplex to an arbitrary polytope $\mP$. A natural approach is to consider the polytope $\mP$ as convex hull of its vertices, i.e., we map a point $x\in\mP$ to a point $\lambda_x\in\mS_{N}$, where $\mV = \{v_1,v_2,...\}$ denotes the set of vertices of $\mP$ and $N=|\mV|$. Given a linear objective $c\in\reals^n$, consider its extension to $\mathbb{R}^N$ given by the vector $c_{ext}\in\reals^{N}$ such that $c_{ext}(i) = v_i \cdot c$, for all $i\in[N]$. Now we can see that
\begin{eqnarray}\label{eq:linear_equiv}
\min_{y\in\mP}y\cdot c \equiv \min_{\lambda\in\mS_{N}}\lambda \cdot c_{ext} .
\end{eqnarray}

Thus, following our approach for the probabilistic simplex, it is tempting to consider as the output of a LLOO for $\mP$, the point $p = \sum_{i=1}^N\lambda^*_iv_i$, where $\lambda^*$ is an optimal solution to the following optimization problem:

\begin{eqnarray}\label{eq:lloo_poly}
&&\min_{\lambda\in\mS_{N}}\lambda \cdot c_{ext}  \nonumber \\
\textrm{s.t. } &&\Vert{\lambda-\lambda_x}\Vert_1 \leq d ,
\end{eqnarray}
where $\lambda_x\in\mS_{N}$ is a mapping of the LLOO input point - $x$ to $\mS_{N}$ and $d$ is a positive scalar. Note that since $\lambda^*\in\mS_N$, the solution $p$ is always a feasible point of the polytope $\mP$.

The main question is whether we can find a value of $d$ such that a solution to Problem \eqref{eq:lloo_poly} indeed corresponds to the output of a LLOO for $\mP$ with a reasonable parameter $\rho$, as in the case of the simplex.

Our implementation of a LLOO for an arbitrary polytope $\mP$ based on solving Problem \eqref{eq:lloo_poly} and outputting the corresponding point in $\mP$ is given below (Algorithm \ref{alg:llo}). The algorithm is a clear extension of Algorithm \ref{alg:llosimplex} for the simplex, and basically moves mass from vertices in the support of the input point $x$ (that is, vertices with non-zero weight in the convex decomposition of $x$) which have large (signed) product with the linear objective $c$, to a single vertex (possibly not in the support of the input point $x$) which minimizes the dot product with $c$. The latter is just the result of calling the linear optimization oracle of the polytope with respect to the linear objective $c$.
 
Note that the algorithm assumes that the input point $x$ is given in the form of a convex combination of vertices of the polytope. Later on we show that maintaining such a decomposition of the input point $x$ is straightforward and efficient when the LLOO is used with any of the optimization algorithms considered in this work. Note also that in the algorithm we implicitly fix the value $d$ in Problem \eqref{eq:lloo_poly} to $d= 2\frac{\sqrt{n}\psi}{\xi}r$ (recall that $\psi,\xi$ are geometric quantities of the polytope at hand, defined formally in Section \ref{sec:prelim}), which is justified by our analysis. 

It is important to note that, as in the case of Algorithm \ref{alg:llosimplex} for the simplex, the running time of Algorithm \ref{alg:llo} does not explicitly depends on the number of vertices $N$, but only on the number of non-zeros in the vector $\lambda$ (the mapping of the input point $x$ to $\mS_N$), the natural dimension of $\mP$ - $n$ and the time to complete a single call to the linear optimization oracle of the polytope - $\oraclep(\cdot)$. In particular, observe that in the computations in lines 3-10 of the algorithm, one needs to consider only the vertices $v_i$ for which $\lambda_i > 0$.

\begin{algorithm}
\caption{Local Linear Optimization Oracle for Polytope $\mP$}
\label{alg:llo}
\begin{algorithmic}[1]
\STATE Input: point $x\in\mP$ such that $x = \sum_{i=1}^{N}\lambda_iv_i$, $\lambda\in{\mS_{N}}$, radius $r > 0$, linear objective $c\in\mathbb{R}^n$
\STATE $\Delta \gets  \min\lbrace{\frac{\sqrt{n}\psi}{\xi}r, 1}\rbrace$
%\STATE Set: $\mV_x \gets \{v_i\in\mV \, | \, \lambda_i > 0\}, \qquad N_x \gets \vert{\mV_x}\vert$
\STATE $\forall{j\in[N]}$: $\ell_i \gets v_i \cdot c$
\STATE Let $i_1,...i_{N}$ be a permutation over $[N]$ such that $\ell_{i_1} \geq \ell_{i_2} \geq ... $
%\STATE Let $i_1,...i_{N_x}$ be an ordering of $[|\mV|]$
\STATE Let $k$ be the smallest integer such that $\sum_{i=1}^j\lambda_i \geq \Delta$
\STATE $p_- \gets \vec{0}$
\FOR{$j=1...k-1$}
\STATE $p_- \gets p_- + \lambda_{i_j}v_{i_j}$
\ENDFOR
\STATE $p_- \gets p_- + \left({\Delta - \sum_{j=1}^{k-1}\lambda_{i_j}}\right)v_{i_k}$
\STATE $v^* \gets \oraclep(c)$
\STATE $p_+ \gets \Delta\cdot v^*$
\RETURN $p \gets x + p_+ - p_-$
\end{algorithmic}
\end{algorithm}

We turn to prove that there is indeed a choice for the parameter $d$ in Problem \eqref{eq:lloo_poly} (the one used to set $\Delta$ in Algorithm \ref{alg:llo}) such that Algorithm \ref{alg:llo} is indeed a LLOO for $\mP$. Towards this end, the main step is to show that there exists a constant $c(\mP)$, such that given a query point $x\in\mP$ in the form $x=\sum_{i=1}^{N}\lambda_x(i)v_i$ where $\lambda_x\in\mS_{N}$, and a point $y\in\mP$, there exists a mapping of $y$ to $\mS_N$, i.e., a point $\lambda_y\in\mS_N$ satisfying $y = \sum_{i=1}^{N}\lambda_y(i)v_i$, such that 
\begin{eqnarray}\label{eq5}
\Vert{\lambda_x-\lambda_y}\Vert_1 \leq c(\mP)\Vert{x-y}\Vert . 
\end{eqnarray}
This fact is a consequence of Lemmas \ref{lem:llo_dist_bound1}, \ref{lem:llo_dist_bound2}. Lemma \ref{lem:llo_dist_bound1} considers a certain way to map a point $y\in\mP$ to $\lambda_y\in\mS_N$ which has useful properties. Lemma \ref{lem:llo_dist_bound2} then builds on these properties to give a consequence in the spirit of Eq. \eqref{eq5} by considering the projection of the vector $(x-y)$ onto a certain set of constraints defining the polytope $\mP$.

\begin{lemma}\label{lem:llo_dist_bound1}
Let $x\in\mP$ and $\lambda\in\mS_{N}$ such that $x = \sum_{i=1}^N\lambda_iv_i$, and let $y\in\mP$. Write $y = \sum_{i=1}^N(\lambda_i-\Delta_i)v_i + (\sum_{i=1}^N\Delta_i)z$ for values $\Delta_i \in [0,\lambda_i] \, \forall{i\in[N]}$ and $z\in\mP$, such that the sum $\Delta = \sum_{i=1}^N\Delta_i$ is minimized. Then, for all $i\in[N]$ for which $\Delta_i > 0$, there exists an index $j_i\in[m]$ such that $A_2(j_i)\cdot v_i < b_2(j_i)$ and $A_2(j_i)\cdot z = b_2(j_i)$.
\end{lemma}

\begin{proof}
By way of contradiction, suppose the lemma is false and let $i'\in[N]$ such that $\Delta_{i'} > 0$ and $\forall{}j\in[m]$ it holds that if $A_2(j)\cdot v_{i'} < b_2(j)$ then $A_2(j)\cdot z < b_2(j)$.  Fixing some $j\in[m]$ we consider two cases. If $A_2(j)\cdot v_{i'} = b_2(j)$ then we have that
\begin{eqnarray}\label{eq3}
\forall\gamma\geq 0: \qquad A_2(j)\cdot(z-\gamma{}v_{i'}) \leq b_2(j) - \gamma{}b_2(j) = (1-\gamma)b_2(j)  .
\end{eqnarray}

On the other hand, if $A_2(j)\cdot v_{i'} < b_2(j)$, then by the assumption we have that $A_2(j)\cdot z <b_2(j)$.  
Denote
\begin{eqnarray*}
\delta_j := b_2(j) - A_2(j)\cdot v_{i'} , \quad \epsilon_j := b_2(j) - A_2(j)\cdot z,
\end{eqnarray*}
and note that $\delta_j >0$ and $\epsilon_j > 0$.

Now it holds that 
\begin{eqnarray}\label{eq4}
\forall\gamma\in[0,\frac{\epsilon_j}{\delta_j}]:  \qquad A_2(j)\cdot(z-\gamma{}v_{i'}) &=& b_2(j) - \epsilon_j -\gamma(b_2(j)-\delta_j) \nonumber \\
&=& (1-\gamma)b_2(j) - (\epsilon_j - \gamma\delta_j) \nonumber \\
&\leq & (1-\gamma)b_2(j) .
\end{eqnarray}
Let $\tilde{\gamma} = \min\{\frac{\epsilon_j}{\delta_j} \, | \, j\in[m], \, A_2(j)\cdot v_{i'} < b_2(j)\}$ (note that by definition $\tilde{\gamma} > 0$, since it is the minimum over a set of strictly positive scalars). % and let $\gamma \in [0, \min\{\tilde{\gamma}, 1\}]$. 

Combining Eq. \eqref{eq3}, \eqref{eq4} for all $j\in[m]$, we have that
\begin{eqnarray*}
\forall \gamma\in[0,\min\{\tilde{\gamma},1\}]: \qquad A_2(z-\gamma{}v_{i'}) \leq (1-\gamma)b_2 .
\end{eqnarray*}

Since $v_{i'},z$ are both feasible, it also holds that $A_1(z-\gamma{}v_{i'}) = (1-\gamma)b_1$ and thus we arrive at the conclusion that
\begin{eqnarray}
\forall \gamma\in[0,\min\{\tilde{\gamma},1\}]: \qquad z-\gamma{}v_{i'}\in(1-\gamma)\mP .
\end{eqnarray}

Thus in particular, by choosing $\gamma\in(0, \min\{\tilde{\gamma}, 1\}]\cap(0,\frac{\Delta_{i'}}{\Delta}]$ (recall that $\Delta_{i'}>0$ and $\tilde{\gamma} > 0$), we have that there exists $w\in\mP$ such that $z =(1-\gamma{})w + \gamma{}v_{i'}$, and
\begin{eqnarray*}
y&=&\sum_{i=1}^N(\lambda_i-\Delta_i)v_i + \Delta{}z \\
&=& \sum_{i=1}^N(\lambda_i-\Delta_i)v_i + \Delta((1-\gamma{})w + \gamma{}v_{i'}) \\
&=& \left({\sum_{i=1, i\neq i'}^N(\lambda_i-\Delta_i)v_i}\right) + (\lambda_{i'} - (\Delta_{i'} - \gamma\Delta) - \gamma\Delta)v_{i'} + \Delta(1-\gamma{})w + \gamma\Delta{}v_{i'} \\
&=& \left({\sum_{i=1, i\neq i'}^N(\lambda_i-\Delta_i)v_i }\right)+ (\lambda_{i'} - (\Delta_{i'} - \gamma\Delta)v_{i'} + \Delta(1-\gamma{})w .
\end{eqnarray*}

Thus, by defining $\forall{i\in[N]}, i\neq{}i'$: $\tilde{\Delta}_i = \Delta_i$ and $\tilde{\Delta}_{i'} = \Delta_{i'} - \gamma\Delta$, we have that $y=\sum_{i=1}^N(\lambda_i-\tilde{\Delta}_i)v_i + (\sum_{i=1}^N\tilde{\Delta}_i)w$ with $\sum_{i=1}^N\tilde{\Delta}_i < \sum_{i=1}^N\Delta_i$, which contradicts the minimality of $\sum_{i=1}^N\Delta_i$.
\end{proof}

In Lemma \ref{lem:llo_dist_bound2} we are going to examine the projection of a vector $(x-y)$ onto a set of constraints of $\mP$ satisfied by a certain feasible point $z\in\mP$. However, we would like that this set will not be too large. The following simple lemma shows that it suffices to consider a basis for the set of constraints satisfied by $z$.

\begin{lemma}\label{lemma:eq_basis}
Let $z\in\mP$ and denote $C(z) = \lbrace{i \in [m] \, | \, A_2(i)\cdot z = b_2(i)}\rbrace$ and let $C_0(z)\subseteq{}C(z)$ be such that the set $\lbrace{A_2(i)}\rbrace_{i\in{}C_0(z)}$ is a basis for the set $\lbrace{A_2(i)}\rbrace_{i\in{}C(z)}$. Then given a point $y\in\mP$, if there exists $i\in{}C(z)$ such that $A_2(i)\cdot y < b_2(i)$, then there exists $i_0\in{}C_0(z)$ such that $A_2(i_0)\cdot y < b_2(i_0)$.
\end{lemma}

\begin{proof}
Fix $z\in\mP$ and let $C(z), C_0(z)$ be as in the lemma. Assume by way of contradiction that there exists $y\in\mP$ and $i\in{}C(z)$ such that $A_2(i)\cdot y < b_2(i)$ and for any $j\in{}C_0(z)$ it holds that $A_2(j)\cdot y = b_2(j)$. Since $A_2(i)$ is a linear combination of vectors from $\lbrace{A_2(j)}\rbrace_{j\in{}C_0(z)}$, there exists scalars $\{\alpha_j\}_{j\in{}C_0(z)}$, not all zeros, such that $A_2(i) = \sum_{j\in C_0(z)}\alpha_jA_2(j)$. From our assumption on $y$ it follows that
\begin{eqnarray*}
b_2(i) > A_2(i)\cdot y = \sum_{j\in C_0(z)}\alpha_jA_2(j)\cdot y = \sum_{j\in C_0(z)}\alpha_jb_2(j) .
\end{eqnarray*}
However, since for all $j\in C(z)$ it holds that $A_2(j)\cdot z = b_2(j)$, we have that
\begin{eqnarray*}
b_2(i) = A_2(i)\cdot z = \sum_{j\in C_0(z)}\alpha_jA_2(j)\cdot z = \sum_{j\in C_0(z)}\alpha_jb_2(j) .
\end{eqnarray*}

Thus we arrive at a contradiction and the lemma follows.
\end{proof}

\begin{lemma}\label{lem:llo_dist_bound2}
Let $x\in\mP$ and $\lambda\in\mS_N$ such that $x=\sum_{i=1}^N\lambda_ix_i$, and let $y\in\mP$. Write $y=\sum_{i=1}^N(\lambda_i-\Delta_i)v_i + (\sum_{i=1}^N\Delta_i)z$, where $\forall i\in[N]: \Delta_i\in[0,\lambda_i]$ and $z\in\mP$, such that the sum $\sum_{i=1}^N\Delta_i$ is minimized (as in Lemma \ref{lem:llo_dist_bound1}). Then it holds that
\begin{eqnarray*}
\sum_{i=1}^N\Delta_i \leq \frac{\sqrt{n}\psi}{\xi}\Vert{x-y}\Vert .
\end{eqnarray*}
As a consequence, $y$ could be mapped to a point $\lambda_y\in\mS_N$ such that 
\begin{eqnarray*}
\Vert{\lambda-\lambda_y}\Vert_1 \leq 2\sum_{i=1}^N\Delta_i \leq 2\frac{\sqrt{n}\psi}{\xi}\Vert{x-y}\Vert .
\end{eqnarray*}
\end{lemma}
\begin{proof}
Denote $C(z) = \lbrace{j\in[m] \, | \, A_2(j)z = b_2(j)}\rbrace$ and note that according to Lemma \ref{lem:llo_dist_bound1} it holds that $C(z) \neq \emptyset$. Let $C_0(z) \subseteq C(z)$ such that the set of vectors $\lbrace{A_2(i)}\rbrace_{i\in{}C_0(z)}$ is a basis for the set $\lbrace{A_2(i)}\rbrace_{i\in{}C(z)}$. Denote by $A_{2,z} \in \mathbb{R}^{\vert{C_0(z)}\vert\times{}n}$ the matrix $A_2$ after deleting every row $i\notin C_0(z)$ and recall that by definition $\Vert{A_{2,z}}\Vert \leq \psi$. Then it holds that
\begin{eqnarray*}
\Vert{x-y}\Vert^2 &=& \Vert{\sum_{i\in[N]: \Delta_i > 0}\Delta_i(v_i - z)}\Vert^2 
\geq \frac{1}{\Vert{A_{2,z}}\Vert^2}\Vert{A_{2,z}\sum_{i\in[N]: \Delta_i > 0}\Delta_i(v_i - z)}\Vert^2 \\
& \geq & \frac{1}{\psi^2}\Vert{\sum_{i\in[N]: \Delta_i > 0}\Delta_iA_{2,z}(v_i - z)}\Vert^2
\\
&=& \frac{1}{\psi^2}\sum_{j\in{}C_0(z)}\left({\sum_{i\in[N]: \Delta_i > 0}\Delta_i{}(A_2(j)\cdot v_i - b_2(j))}\right)^2 .
\end{eqnarray*}
Note that $\vert{C_0(z)}\vert \leq n$ and that for any vector $x\in\mathbb{R}^{\vert{C_0(z)}\vert}$ it holds that $\Vert{x}\Vert \geq \frac{1}{\sqrt{\vert{C_0(z)}\vert}}\Vert{x}\Vert_1$. Thus we have that
\begin{eqnarray*}
\Vert{x-y}\Vert^2 
&\geq & \frac{1}{n\psi^2}\left({\sum_{j\in{}C_0(z)}\left|{\sum_{i\in[N]: \Delta_i > 0}\Delta_i{}(A_2(j)\cdot v_i - b_2(j))}\right|}\right)^2 \\
&=& \frac{1}{n\psi^2}\left({\sum_{j\in{}C_0(z)}\sum_{i\in[N]: \Delta_i > 0}\Delta_i{}(b_2(j) - A_2(j)\cdot v_i)}\right)^2 .
\end{eqnarray*}

Combining Lemma \ref{lem:llo_dist_bound1} and Lemma \ref{lemma:eq_basis}, we have that for all $i\in[N]$ such that $\Delta_i > 0$ there exists $j\in{}C_0(z)$ such that $A_2(j)\cdot v_i \leq b_2(j)-\xi$. Hence,
\begin{eqnarray*}
\Vert{x-y}\Vert^2 &\geq & \frac{1}{n\psi^2}\left({\sum_{i\in[N]: \Delta_i > 0}\Delta_i{}\xi}\right)^2 
= \frac{\xi^2}{n\psi^2}\left({\sum_{i\in[N]: \Delta_i > 0}\Delta_i{}}\right)^2 .
\end{eqnarray*}

Thus we conclude that $\sum_{i=1}^N\Delta_i \leq \frac{\sqrt{n}\psi}{\xi}\Vert{x-y}\Vert$.
\end{proof}

The following lemma establishes that Algorithm \ref{alg:llo} is indeed a local linear optimization oracle for $\mP$ with parameter $\rho = \sqrt{n}\mu$ ($\mu$ is a geometric parameter of $\mP$ that was formally defined in Section \ref{sec:prelim}).

\begin{lemma}\label{lem:oracle}
Let $p$ be the point returned by algorithm \ref{alg:llo} when called with the input $x = \sum_{i=1}^N\lambda_iv_i$, $r$, $c$.
Then the following conditions hold:
\begin{enumerate}
\item $p\in\mP$.
\item $\Vert{x - p}\Vert \leq \sqrt{n}\mu{}r$.
\item $\forall{}y\in{}\ball_r(x)\cap\mP$ it holds that $c\cdot y \geq c \cdot p$.
\end{enumerate}
\end{lemma}
\begin{proof}
Condition 1. holds since $p$ is clearly given as a convex combination of points in $\mathcal{V}$. 
For conditions 2,3, note that we can write the returned point $p$ as $p=\sum_{i=1}^N(\lambda_i-\Delta_i)v_i + \Delta{}v^*$, where $\Delta$ is as in Algorithm \ref{alg:llo},  for all $i\in[N]: \Delta_i\in[0,\lambda_i]$, $\sum_{i=1}^N\Delta_i = \Delta$ and $v^*\in\mV$.
Thus we have that
\begin{eqnarray*}
\Vert{x-p}\Vert &=& \Vert{\sum_{i=1}^N\Delta_iv_i - \Delta{}v^*}\Vert
= \Vert{\sum_{i=1}^N\Delta_i(v_i -v^*)}\Vert \\
& \leq & \sum_{i=1}^N\Delta_i\Vert{v_i-v^*}\Vert \leq \Delta\mD \leq \frac{\sqrt{n}\psi\mD}{\xi} = \sqrt{n}\mu ,
\end{eqnarray*}
which gives condition 2.

Finally, for condition 3, note that from Algorithm \ref{alg:llo} and Lemma \ref{lem:simplex_opt} it follows that the returned point $p$ could be written as $p = \sum_{i=1}^N\lambda^*(i)v_i$ such that $\lambda^* $ is an optimal solution to Problem \eqref{eq:lloo_poly} with parameter $d = 2\frac{\sqrt{n}\psi}{\xi}r$. From Lemma \ref{lem:llo_dist_bound2} we have that for any point $y\in\ball_r(x)\cap\mP$, $y$ could be mapped to a point $\lambda_y\in\mS_N$ such that $\Vert{\lambda-\lambda_y}\Vert_1 \leq 2\frac{\sqrt{n}\psi}{\xi}r$. Thus we have that
\begin{eqnarray*}
p\cdot c &=& \sum_{i=1}^N\lambda^*(i)v_i\cdot c = \lambda^* \cdot c_{ext} = \min_{\lambda_z\in\mS_N: \, \Vert{\lambda - \lambda_z}\Vert_1 \leq 2\frac{\sqrt{n}\psi}{\xi}r}\lambda_z \cdot c_{ext} \leq \lambda_y \cdot c_{ext} \\
&=& \sum_{i=1}^N\lambda_y(i)v_i\cdot c = y \cdot c .
\end{eqnarray*}
\end{proof}

\subsubsection{Maintaining a small decomposition of the input point $x$ and efficient implementation of Algorithm \ref{alg:llo}}

Algorithm \ref{alg:llo} assumes that the input point $x$ is given by its convex decomposition into vertices. All
optimization algorithms in this work use Algorithm \ref{alg:llo} in the following way: they give as input to Algrotihm \ref{alg:llo} the current feasible iterate $x_t\in\mP$, and then given the output of Algorithm \ref{alg:llo}, denoted in all algorithms by $p_t$, they produce the next iterate $x_{t+1}$ by taking a convex combination $x_{t+1} \gets  (1-\alpha)x_t + \alpha{}p_t$, for some parameter $\alpha\in[0,1]$. Note that Algorithm \ref{alg:llo} implicitly produces the convex decomposition of the returned point $p_t$ and thus, given the convex decomposition of $x_t$, updating it to the convex decomposition of $x_{t+1}$ is straightforward. 

Moreover, denoting $\mathcal{V}_t \subseteq \mathcal{V}$ the set of vertices that forms the convex decomposition of $x_t$ (i.e. the vertices with non-zero weight in the decomposition), it is clear from Algorithm \ref{alg:llo} and the discussion above that $\vert{\mathcal{V}_{t+1} \setminus \mathcal{V}_t}\vert \leq 1$, since at most a single vertex ($v^*$ in Algorithm \ref{alg:llo}) is added to the decomposition. 

This brings us to the following lemma.
\begin{lemma}\label{lem:llo_complex}
Algorithm \ref{alg:llo} admits an implementation such that each invocation of the algorithm requires a single call to the oracle $\oraclep$ and additional $O(T(n+\log{T}))$ time, where $T$ is the overall number of calls to the algorithm.
\end{lemma}

\begin{proof}
Clearly Algorithm \ref{alg:llo} calls the oracle $\oraclep$ only once per invocation. The complexity of all other operations depends on the number of non-zeros in the vector $\lambda$, i.e., the number of vertices in the convex decomposition of the input point $x$. As discussed above, if we denote by $x_t, x_{t+1}$ the inputs to the $t$ and $t+1$ times Algorithm \ref{alg:llo} was invoked respectively, and by $N_t, N_{t+1}$ the number of vertices in the convex decomposition of $x_t, x_{t+1}$ respectively, then $N_{t+1} \leq N_t+1$. Thus, if the algorithm is invoked for a total number of $T$ times and the initial point - $x_1$ is a vertex of $\mP$, then at any time $t\in[T]$ it holds that  $N_t \leq T$. Since all other operations except for calling $\oraclep$ consist of computing $N_t$ inner products between vectors in $\mathbb{R}^n$ and sorting $N_t$ scalars, the lemma follows.
\end{proof}

Note that we can get rid of the linear dependence on $T$ in the bound in lemma \ref{lem:llo_complex} by decomposing the iterate $x_t$ into a convex sum of fewer vertices in case the number of vertices in the current decomposition - $N_t$ becomes too large. From Carath\'eodory's theorem we know that we can find such a decomposition with at most $n+1$ vertices.
Moreover, for many polytopes of interest (such as the flow polytope), there is an even more efficient algorithm for computing such a decomposition (however these are beyond the scope of this paper). It follows from previous discussions that we will need to invoke such a decomposition procedure only every $O(n)$ iterations which will keep the amortized iteration complexity low.

Another generic approach to the above problem that relies only on the use of the linear optimization oracle - $\oraclep$ (and which might also be more efficient), is to ``bootstrap" Algorithm \ref{alg:offline} to compute a more compact decomposition of $x_t$.
If $N_t$ is too large we can compute a new decomposition of the input point $x_t$ by solving the optimization problem $\min_{y\in\mP}\Vert{x_t-y}\Vert^2$ up to some precision $r_t^2$, where $r_t$ is the current radius parameter of the LLOO. Using Theorem \ref{thm:offthm}, the result will be a point $\tilde{x}_t$ given by a decomposition into $O(n\mu^2\log(1/r_t))$ vertices of $\mP$ such that $\Vert{\tilde{x}_t-x_t}\Vert \leq r_t$. Now, by using the point $\tilde{x}_t$ to maintain the input to the LLOO and replacing the input $r_t$ to LLOO with $\tilde{r}_t = 2r_t$ we get (via the triangle inequality) a modified LLOO with parameter $\tilde{\rho} = 2\rho+1=2\sqrt{n}\mu+1$. As discussed above, we will need to invoke this decomposition procedure only every $O(n\mu^2\log(1/r_t))$ iterations which leads to the following lemma.

\begin{lemma}\label{lem:llo_complex2}
Assume that on every invocation of the LLOO algorithm, the input $r$ to the LLOO is lower-bounded by some $r_0>0$. Then there exists an implementation for a LLOO with parameter $\rho=2\sqrt{n}\mu+1$, such that the amortized linear optimization oracle complexity per iteration is 2, and the additional amortized complexity per iteration is $O\left({n\mu^2\log(1/r_0)\left({n+\log(n\mu^2\log(1/r_0)}\right)}\right)$.
\end{lemma}
The proof follows the same lines as that of Lemma \ref{lem:llo_complex}.

We note that in our online algorithms the lower bound $r_0$ in Lemma \ref{lem:llo_complex2} will always satisfy: $\log(1/r_0) = O(\log(T))$, where $T$ is the overall length of the game, and thus the running time per iteration will depend only logarithmically on $T$. 

It is also worth mentioning that we can significantly accelerate Algorithm \ref{alg:llo} by using parallel computations. Note that all dot product computations in line 3 of the algorithm (recall again that in practice we need to carry out these computations only for vertices $v_i$ for which $\lambda_i > 0$) are independent of each other and could be computed in parallel.

\section{Online and Stochastic Convex Optimization}\label{sec:online}

In this section we present algorithms for the general setting of online convex optimization that are suitable when the decision set is a polytope. We present regret bounds for both general convex losses and for strongly convex losses. These regret bounds imply convergence rates for stochastic convex optimization and non-smooth convex optimization over polyhedral sets as described in subsections \ref{sec:stoc_opt}, \ref{sec:nonsmooth_opt}. In the sequel we also present an algorithm for the bandit setting.

Our algorithm for online convex optimization in the full information setting is given below (Algorithm \ref{alg:rftlfw}) . The algorithm is based on the ideas presented in Subsection \ref{sec:RFTL}, i.e., iteratively approximating the steps of a regret-optimal algorithm known as \textit{Regularize Follow the Leader} using the update step of our Algorithm \ref{alg:offline}, which amounts to a single call a local linear optimization oracle (which in turn, given the construction presented in Section \ref{sec:llo}, amounts to a single call to the linear optimization oracle of the polytope).

For ease of presentation, we use a standard assumption that the algorithm has knowledge on several parameters of the problem including the length of the game - $T$, an upper bound on the magnitude of the gradients of the observed loss functions - $G$, and a lower bound on the strong convexity of the observed functions - $\sigma$ (which may also be zero) \footnote{in case one of these bounds is unknown, one can use standard techniques such as the well known ``doubling trick", which increases the overall regret only by a log factor.} .

\begin{algorithm}
\caption{LLOO-based Online Convex Optimization}
\label{alg:rftlfw}
\begin{algorithmic}[1]
\STATE Input: horizon $T$, upper bound on gradients $G$, strong convexity parameter $\sigma$, $\mathcal{A}(x,r,c)$ - LLOO with parameter $\rho$ for $\mP$
\STATE Set: $\alpha \gets  \left\{ \begin{array}{ll}
         (3\rho^2)^{-1} & \mbox{if $\sigma = 0$}\\
         (5\rho^2)^{-1}& \mbox{if $\sigma > 0$}\end{array} \right.$
\STATE Set: $\eta \gets \frac{D}{18G\rho\sqrt{T}}, \qquad T_0 \gets (25\rho^2)^2$
\STATE Let $x_1$ be an arbitrary vertex in $\mathcal{V}$
\FOR{$t = 1...T$}
\STATE Play $x_t$
\STATE Receive $f_t(x)$
\STATE Define the function:
\begin{eqnarray*}
F_t(x) :=  \left\{ \begin{array}{ll}
         \eta\left({\sum_{\tau=1}^t\nabla{}f_{\tau}(x_{\tau})\cdot x}\right) + \Vert{x-x_1}\Vert^2 & \mbox{if $\sigma = 0$}\\
        \left({\sum_{\tau=1}^t\nabla{}f_{\tau}(x_{\tau})\cdot x + \frac{\sigma}{2}\Vert{x-x_{\tau}}\Vert^2}\right) + T_0\frac{\sigma}{2}\Vert{x-x_1}\Vert^2 & \mbox{if $\sigma > 0$}\end{array} \right.
\end{eqnarray*}
\STATE Set:
\begin{eqnarray*}
r_t \gets  \left\{ \begin{array}{ll}
         \frac{D}{\sqrt{T}}\left({\rho + \frac{1}{18\rho}}\right) & \mbox{if $\sigma = 0$}\\
        \frac{2(G+\sigma{}D)}{\sigma(t+T_0)}(60\rho^2+1) & \mbox{if $\sigma > 0$}\end{array} \right.
\end{eqnarray*}
\STATE $p_t \gets \mathcal{A}(x_t,r_t,\nabla{}F_t(x_t))$
\STATE $x_{t+1} \gets x_t + \alpha(p_t - x_t)$
\ENDFOR
\end{algorithmic}
\end{algorithm}

We prove the following two main theorems.

Denote $G = \sup_{x\in\mP,t\in[T]}\Vert{\nabla{}f_t(x)}\Vert$ and recall that we have a construction for a local linear optimization oracle with parameter $\rho = O(\sqrt{n}\mu)$ for the decision set $\mP$.

\begin{theorem}\label{thr:mainthr}
In case Algorithm \ref{alg:rftlfw} is instanciated with the LLOO described in Section \ref{sec:llo} (Algorithm \ref{alg:llo}), then for arbitrary convex loss fundtions, the regret of the algorithm is $O(GD\mu\sqrt{nT})$.
\end{theorem}

\begin{theorem}\label{thm:sc}
In case Algorithm \ref{alg:rftlfw} is instanciated with the LLOO described in Section \ref{sec:llo} (Algorithm \ref{alg:llo}), then for $\sigma$-strongly convex loss functions, the regret of the algorithm is $O(\sigma{}D^2\rho^4 + (G+\sigma{}D)^2n\mu^2/\sigma)\log{T})$.
\end{theorem}

Applying the above two theorems with the reduction of stochastic optimization to online optimization described in Subsection \ref{sec:stoc_opt}, yields the following two corollaries.

\begin{corollary}
Let $F(x)=\mathbb{E}_{f\sim\mathcal{D}}[f(x)]$, where $\mathcal{D}$ is a distribution over arbitrary convex functions, and assume the availability of an oracle $\mO_\mD$ for sampling functions from the distribution $\mD$. Then running Algorithm \ref{alg:rftlfw}, instanciated with the LLOO described in Section \ref{sec:llo}, with a sequence of $T$ loss functions sampled i.i.d. using $\mO_{\mD}$ and denoting $\bar{x}_T = \frac{1}{T}\sum_{t=1}^Tx_t$ (the average of iterates), we have that $$\E[F(\bar{x}_T)] - \min_{x^* \in \mP} F(x^*) = O\left({\frac{GD\sqrt{n}\mu}{\sqrt{T}}}\right) .$$
\end{corollary}

\begin{corollary}
Let $F(x)=\mathbb{E}_{f\sim\mathcal{D}}[f(x)]$, where $\mathcal{D}$ is a distribution over $\sigma$-strongly convex functions, and assume the availability of an oracle $\mO_\mD$ for sampling functions from the distribution $\mD$. Then running Algorithm \ref{alg:rftlfw}, instanciated with the LLOO described in Section \ref{sec:llo}, with a sequence of $T$ loss functions sampled i.i.d. using $\mO_{\mD}$ and denoting $\bar{x}_T = \frac{1}{T}\sum_{t=1}^Tx_t$ (average of iterates), we have that $$\E[F(\bar{x}_T)] - \min_{x^* \in \mP} F(x^*) = O\left({\frac{\sigma^2D^2\rho^4 + (G+\sigma{}D)^2n\mu^2\log(T)}{\sigma{}T}}\right) .$$
\end{corollary}

In the following two subsections we prove Theorems \ref{thr:mainthr}, \ref{thm:sc}.

\subsection{Analysis for general convex losses}

%For time $t\in[T]$ we define the function $F_t(x) = \eta\left({\sum_{\tau=1}^t\nabla{}f_{\tau}(x_{\tau})^{\top}x}\right) + \Vert{x-x_1}\Vert^2$ where $\eta$ is a parameter that will by determined in the analysis.

In this subsection we analyze the regret of Algorithm \ref{alg:rftlfw} in case the observed loss functions are all convex but not necessarily strongly convex, that is, $\sigma = 0$.

Consider the sequence of points $\lbrace{x^*_t}\rbrace_{t=1}^{T+1}$ such that for all $t\in[T+1]$, $x^*_t = \arg\min_{x\in\mP}F_{t-1}(x)$, where for all $t\in[T]$, $F_t(x)$ is as defined in Algorithm \ref{alg:rftlfw} (for $\sigma = 0)$ and for $t=0$ we define $F_0(x):=\Vert{x-x_1}\Vert^2$ . The regret analysis is comprised of two parts. Part 1 shows that on any time $t$, the point $x_t$ played by Algorithm \ref{alg:rftlfw} is close to the corresponding point $x^*_t$. Thus by a Lipschitz argument, the cumulative loss of the sequence $\lbrace{x_t}\rbrace_{t=1}^T$ is close to that of $\lbrace{x^*_t}\rbrace_{t=1}^T$. Part 2 then follows the analysis of an algorithm known as \textit{Regularized Follow the Leader} (see \cite{Hazan09}) to claim that the sequence of points $\lbrace{x^*_t}\rbrace_{t=1}^T$ achieves low regret with respect to the sequence of observed loss functions. 

%Denote $x_1^* = x_1$ and for all $t\in[T-1]$ $x_{t+1}^* = \arg\min_{x\in\mP}F_t(x)$. Denote also $x^* = \arg\min_{x\in\mP}\sum_{t=1}^Tf_t(x)$. Observe that $F_t(x)$ is $1$-smooth and $1$-strongly convex.

\begin{lemma}\label{regretlemma}
Fix $\epsilon > 0$. Let
\begin{eqnarray*}
\eta = \frac{\sqrt{\epsilon}}{18G\rho^2}, \qquad \alpha = \frac{1}{3\rho^2}, \qquad r_t = \sqrt{\epsilon}+\eta{}G \quad \forall t\in[T] .
\end{eqnarray*}

%There exists an explicit choice for the parameters $\eta, \alpha, \lbrace{r_t}\rbrace_{t=1}^T$ such that 
Then, the sequence of points $\lbrace{x_t}\rbrace_{t=1}^T$ produced by Algorithm \ref{alg:rftlfw} satisfies
that for all $t\in[T]$, $\Vert{x_t - x_t^*}\Vert \leq \sqrt{\epsilon}$.
\end{lemma}
\begin{proof}
Observe that on any time $t\in\{0,1,...,T\}$ it holds that the function $F_t(x)$ is $2$-strongly convex and $2$-smooth.

We prove by induction that for all $t\in[T]$ it holds that $F_{t-1}(x_t) - F_{t-1}(x_t^*) \leq \epsilon$. By the strong-convexity of $F_{t-1}$ (Eq. \ref{eq:strongconvexdist}) this yields that $\Vert{x_t - x_t^*}\Vert \leq \sqrt{\epsilon}$.

The proof is by induction on $t$. For $t=1$ it holds that $x_1 = x_1^*$ and thus the claim holds. Assume now that for time $t \geq 1$ it holds that $F_{t-1}(x_t) - F_{t-1}(x_t^*) \leq \epsilon$. By the strong-convexity of $F_{t-1}(x)$ and the induction hypothesis we have that
\begin{eqnarray}\label{ie:regretlemma:1}
\Vert{x_t - x_t^*}\Vert \leq \sqrt{\epsilon} .
\end{eqnarray}

By the definition of $F_t(x)$ and the optimality of $x_t^*$ we have that 
\begin{eqnarray*}
F_{t}(x_{t}^*) - F_{t}(x_{t+1}^*) &=& F_{t-1}(x_{t}^*) - F_{t-1}(x_{t+1}^*) + \eta\nabla{}f_t(x_t)\cdot(x_t^*-x_{t+1}^*)\\
& \leq & \eta{}G\Vert{x_{t+1}^* - x_t^*}\Vert ,
\end{eqnarray*}
 and thus again by the strong convexity of $F_t(x)$ we have that
\begin{eqnarray}\label{ie:regretlemma:2}
\Vert{x_{t+1}^* - x_t^*}\Vert \leq \eta{}G .
\end{eqnarray}

Combining Eq. \eqref{ie:regretlemma:1}, \eqref{ie:regretlemma:2} we have that
\begin{eqnarray*}
\Vert{x_t - x_{t+1}^*}\Vert \leq \sqrt{\epsilon} + \eta{}G .
\end{eqnarray*}

Using again the induction hypothesis, we have that
\begin{eqnarray}\label{ie:regretlemma:3}
F_{t}(x_t) - F_{t}(x_{t+1}^*) &=& F_{t-1}(x_t) - F_{t-1}(x_{t+1}^*) + \eta\nabla{}f_t(x_t)\cdot(x_t - x_{t+1}^*) \nonumber \\
&\leq & \epsilon + \eta{}G\Vert{x_t - x_{t+1}^*}\Vert 
\leq \epsilon + \eta{G}\sqrt{\epsilon} + \eta^2G^2 .
\end{eqnarray}

Setting $r_t = \sqrt{\epsilon}+\eta{}G$, we can apply Lemma \ref{lem:fw} with respect to $F_t(x), x_t, r_t$ and get,
\begin{eqnarray*}
F_{t}(x_{t+1}) - F_{t}(x_{t+1}^*) \leq  (1-\alpha)(F_{t}(x_{t}) - F_{t}(x_{t+1}^*)) + \alpha^2\rho^2\left({\sqrt{\epsilon}+\eta{}G}\right)^2 .
\end{eqnarray*}

Plugging Eq. \eqref{ie:regretlemma:3} we have that
\begin{eqnarray*}
F_{t}(x_{t+1}) - F_{t}(x_{t+1}^*) &\leq & (1-\alpha)\left({\epsilon + \eta{G}\sqrt{\epsilon} + \eta^2G^2}\right)
+ 2\alpha^2\rho^2\left({\epsilon +\eta{}G\sqrt{\epsilon} + \eta^2G^2}\right) \\
&=& \left({\epsilon + \eta{G}\sqrt{\epsilon} + \eta^2G^2}\right)(1-\alpha + 2\alpha^2\rho^2).
\end{eqnarray*}

Setting $\alpha = \frac{1}{3\rho^2}$ we get that
\begin{eqnarray*}
F_{t}(x_{t+1}) - F_{t}(x_{t+1}^*) \leq \left({\epsilon +\eta{}G\sqrt{\epsilon} + \eta^2G^2}\right)\left({1-\frac{1}{9\rho^2}}\right) .
\end{eqnarray*}

Finally, plugging $\eta = \frac{\sqrt{\epsilon}}{18G\rho^2}$ gives
\begin{eqnarray*}
F_{t}(x_{t+1}) - F_{t}(x_{t+1}^*) \leq \epsilon\left({1+\frac{1}{9\rho^2}}\right)\left({1-\frac{1}{9\rho^2}}\right) < \epsilon .
\end{eqnarray*}
\end{proof}

We also need the following lemma, originally proved in \cite{KalaiVempala}, that states that playing on each time $t$ the point in $\mP$ that minimizes the loss up to time $t$ (including), yields zero regret. A proof is given in the appendix for completeness.

\begin{lemma}\label{lem:be_the_leader}
Let $\lbrace{f_t(x)}\rbrace_{t=1}^T$ be a sequence of loss functions and let $\lbrace{w^*_t}\rbrace_{t=1}^T$ be a sequence of points such that for all $t\in[T]$, $w^*_t \in\arg\min_{w\in\mP}\sum_{\tau=1}^tf_{\tau}(w)$. Then it holds that
\begin{eqnarray*}
\sum_{t=1}^Tf_t(w^*_t) - \min_{w\in\mP}\sum_{t=1}^Tf_t(w) \leq 0 .
\end{eqnarray*}
\end{lemma}

We are now ready to prove Theorem \ref{thr:mainthr}.

\begin{proof}
Denote $x^* \in \arg\min_{x\in\mP}\sum_{t=1}^Tf_t(x)$.
We define a sequence of functions $\lbrace{\tilde{f}_t(x)}\rbrace_{t=1}^T$ as follows: $\tilde{f}_1(x) = \nabla{}f_1(x_1)\cdot x + \frac{1}{\eta}\Vert{x-x_1}\Vert^2$ and $\tilde{f}_t(x) = \nabla{}f_t(x_t)\cdot x$ for all $t \geq 2$. Note that for all $t\in[T]$ we have that $F_t(x) = \eta\sum_{\tau=1}^t\tilde{f}_{\tau}(x)$. Recall that the sequence of points $\lbrace{x^*_t}\rbrace_{t=1}^{T+1}$ satisfies that $x^*_t = \arg\min_{x\in\mP}F_{t-1}(x)$. Hence, for all $t\in[T]$, $x^*_{t+1} = \arg\min_{x\in\mP}F_t(x)$.
Thus, by Lemma \ref{lem:be_the_leader} we have that
\begin{eqnarray*}
\sum_{t=1}^T\tilde{f}_t(x^*_{t+1}) - \sum_{t=1}^T\tilde{f}_t(x^*)
= \sum_{t=1}^T\nabla{}f(x_t)\cdot (x^*_{t+1} - x^*) + \frac{1}{\eta}(\Vert{x_{2}^*-x_1}\Vert^2 - \Vert{x^*-x_1}\Vert^2)
 \leq 0 .
\end{eqnarray*}

Rearranging and using $\Vert{x^*-x_1}\Vert \leq D$ we have that
\begin{eqnarray}\label{eq:6}
\sum_{t=1}^T\nabla{}f_t(x_t)\cdot (x_{t+1}^* - x^*) \leq  \frac{D^2}{\eta} .
\end{eqnarray}

Fix $t\in[T]$. Since $F_t(x)$ is $2$-strongly convex, using Eq. \eqref{eq:strongconvexdist} we have that
\begin{eqnarray}\label{eq:7}
\Vert{x_t^* - x_{t+1}^*}\Vert^2 &\leq &F_{t}(x_t^*) - F_{t}(x_{t+1}^*) = F_{t-1}(x_t^*) - F_{t-1}(x_{t+1}^*) + \eta\nabla{}f_t(x_t)\cdot(x_t^* - x_{t+1}^*) \nonumber \\
&\leq & \eta{}G\Vert{x_t^* - x_{t+1}^*}\Vert,
\end{eqnarray}
where the last inequality follows from the optimality of $x_t^*$ with respect to $F_{t-1}(x)$ and the Cauchy-Schwartz inequality.

Combining Eq. \eqref{eq:6} and Eq. \eqref{eq:7} for all $t\in[T]$ via the Cauchy-Schwartz inequality we have that
\begin{eqnarray*}
\sum_{t=1}^T\nabla{}f_t(x_t) \cdot (x_{t}^* - x^*) \leq \frac{D^2}{\eta} + T\eta{}G^2 .
\end{eqnarray*}

Rearranging and using the Cauchy-Schwartz inequality again we have that
\begin{eqnarray*}
\sum_{t=1}^T\nabla{}f_t(x_t) \cdot (x_{t} - x^*) \leq \frac{D^2}{\eta} + T\eta{}G^2 + G\sum_{t=1}^T\Vert{x_t - x^*_t}\Vert .
\end{eqnarray*}

Fix $\epsilon = \frac{(D\rho)^2}{T}$. Applying Lemma \ref{regretlemma} with respect to our choice of $\epsilon$
and setting $\eta$ accordingly (and recalling that $\rho \geq 1$), we have that
\begin{eqnarray*}
\sum_{t=1}^Tf_t(x_t) - f_t(x^*) \leq \sum_{t=1}^T\nabla{}f_t(x_t)\cdot(x_t - x^*) = O(GD\rho\sqrt{T}),
\end{eqnarray*}
where the first inequality follows from convexity of each $f_t(x)$. 
The theorem now follows since according to our results from Section \ref{sec:llo} we can assume that $\rho = \sqrt{n}\mu$.

\end{proof}

\subsection{Analysis for strongly convex losses}
Here we analyze the regret of Algorithm \ref{alg:rftlfw} in case all loss function are at least $\sigma$-strongly convex for some $\sigma > 0$. The analysis goes along the same lines as the analysis for the non-strongly convex case, but requires a few modifications.

As in the previous subsection we define the sequence $\lbrace{x^*_t}\rbrace_{t=1}^{T+1}$ such that $x^*_t = \arg\min_{x\in\mP}F_{t-1}(x)$, where 
$F_t(x)$ for $t\in\{1,...,T\}$ is defined as in Algorithm \ref{alg:rftlfw} (for $\sigma > 0$) and in addition we define $F_0 := T_0\frac{\sigma}{2}\Vert{x-x_1}\Vert^2$.

\begin{lemma}\label{lemma:liphs}
For any $t\in[T]$, the function $\tilde{f}_t(x) = \nabla{}f_t(x_t)\cdot x + \frac{\sigma}{2}\Vert{x-x_t}\Vert^2$ is $L=G+\sigma{}D$ Lipschitz over $\mP$.
\end{lemma}
\begin{proof}
Fix two points $y,z\in\mP$. Since $\tilde{f}_t(x)$ is convex we have that
\begin{eqnarray*}
\tilde{f}_t(y)-\tilde{f}_t(z) &\leq & \nabla\tilde{f}_t(y) \cdot (y-z) = (\nabla{}f_t(x_t) + \sigma(y-x_t))\cdot(y-z) \\
&\leq & \Vert{\nabla{}f_t(x_t) + \sigma(y-x_t)}\Vert \cdot \Vert{y-z}\Vert
\leq (G+\sigma{}D)\Vert{y-z}\Vert,
\end{eqnarray*}
where the last inequality uses the triangle inequality and the upper bounds $G,D$ for $\Vert{\nabla{}f_t(x_t)}\Vert$ and $\Vert{y-x_t}\Vert$ respectively.
Since the above inequality is symmetric in $y,z$, the lemma follows.
\end{proof}

The following Lemma is analogues to Lemma \ref{regretlemma} for the non-strongly convex case.
\begin{lemma}\label{sc_regretlemma}
Let $L = G + \sigma{}D$, $\alpha = \frac{1}{5\rho^2}$ and $T_0 = (25\rho^2)^2$.
Let $\lbrace{\epsilon_t}\rbrace_{t=1}^T$ be a sequence of positive reals such that $\epsilon_t = \frac{(60\rho^2L)^2}{\sigma(t+T_0)}$ $\forall t\in[T]$. Let
\begin{eqnarray*}
r_t = \sqrt{\frac{4\epsilon_t}{\sigma(t+T_0)}} + \frac{2L}{\sigma(t+T_0)}  \quad \forall t\in[T] .
\end{eqnarray*}
%There is a choice for the parameters $\alpha, r_t, T_0$ such that 
Then, for any $t\in[T]$ it holds that %$\Vert{x_t - x_t^*}\Vert \leq \frac{100\rho^2L}{Ht}$.
$\Vert{x_t - x_t^*}\Vert \leq \sqrt{\frac{\epsilon_t}{\sigma(t-1+T_0)}}$.
\end{lemma}

\begin{proof}
The proof is similar to that of Lemma \ref{regretlemma}. 
Observe that on any time $t\in\{0,1,...,T\}$ it holds that the function $F_t(x)$ is $\sigma(t+T_0)$-strongly convex and $\sigma(t+T_0)$-smooth.

We prove that for any time $t\in[T]$ it holds that $F_{t-1}(x_t) - F_{t-1}(x_t^*) \leq \epsilon_t$, which by the strong convexity of $F_{t-1}(x)$ (see Eq. \eqref{eq:strongconvexdist}) implies that $\Vert{x_t - x_t^*}\Vert \leq \sqrt{\frac{2\epsilon_t}{\sigma(t-1+T_0)}}$.

Clearly for time $t=1$ the claim holds since $x_1 = x_1^*$. Assume that on time $t\geq 1$ it holds that $F_{t-1}(x_t) - F_{t-1}(x_t^*) \leq \epsilon_t$. By the strong convexity of $F_{t-1}(x)$ we again have that
\begin{eqnarray}\label{ie:sc_regretlemma_1}
\Vert{x_t - x_t^*}\Vert \leq \sqrt{\frac{2\epsilon_t}{\sigma(t-1+T_0)}} .
\end{eqnarray}

Define the function $\tilde{f}_t(x) = \nabla{}f_t(x_t) \cdot x + \frac{\sigma}{2}\Vert{x-x_t}\Vert^2$.
It holds that
\begin{eqnarray*}
F_{t}(x_t^*) - F_{t}(x_{t+1}^*) &=& F_{t-1}(x_t^*) - F_{t-1}(x_{t+1}^*) 
+\tilde{f}_t(x_t^*) - \tilde{f}_t(x_{t+1}^*) \\
&\leq &\tilde{f}_t(x_t^*) - \tilde{f}_t(x_{t+1}^*) \leq  L\Vert{x_t^*-x_{t+1}^*}\Vert,
\end{eqnarray*}
where the first inequality follows from the optimality of $x^*_t$ with respect to $F_{t-1}(x)$ and the second inequality follows from Lemma \ref{lemma:liphs}.

By the strong convexity of $F_t(x)$ we thus have that
\begin{eqnarray}\label{ie:sc_regretlemma_2}
\Vert{x_t^* - x_{t+1}^*}\Vert \leq \frac{2L}{\sigma(t+T_0)} .
\end{eqnarray}

Combining Eq. \eqref{ie:sc_regretlemma_1}, \eqref{ie:sc_regretlemma_2} via the triangle inequality we have that
\begin{eqnarray}\label{ie:sc_regretlemma:3}
\Vert{x_t - x_{t+1}^*}\Vert &\leq & \sqrt{\frac{2\epsilon_t}{\sigma(t-1+T_0)}} + \frac{2L}{\sigma(t+T_0)} \nonumber \\
& \leq & \sqrt{\frac{4\epsilon_t}{\sigma(t+T_0)}} + \frac{2L}{\sigma(t+T_0)},
\end{eqnarray}
where the second inequality holds since $T_0 \geq 1$.

Using the induction hypothesis we have that

\begin{eqnarray}\label{ie:sc_regretlemma:4}
F_{t}(x_t) - F_{t}(x_{t+1}^*) &=& F_{t-1}(x_t) - F_{t-1}(x_{t+1}^*) + \tilde{f}_t(x_t) - \tilde{f}_t(x_{t+1}^*) \nonumber \\
&\leq & \epsilon_t + \sqrt{\frac{4L^2\epsilon_t}{\sigma(t+T_0)}} + \frac{2L^2}{\sigma(t+T_0)} ,
\end{eqnarray}
where the inequality follows from Eq. \eqref{ie:sc_regretlemma:3} and Lemma \ref{lemma:liphs}.

Setting $r_t$ to equal the RHS of Eq. \eqref{ie:sc_regretlemma:3}, and applying Lemma \ref{lem:fw} with respect to $F_t(x)$ we have that
\begin{eqnarray*}
F_{t}(x_{t+1}) - F_{t}(x_{t+1}^*) &\leq & 
(1-\alpha)(F_{t}(x_{t}) - F_{t}(x_{t+1}^*)) + \frac{\sigma}{2}(t+T_0)\alpha^2\rho^2r_t^2 \\ &\leq & 
(1-\alpha)\left({\epsilon_t + \sqrt{\frac{4L^2\epsilon_t}{\sigma(t+T_0)}} + \frac{2L^2}{\sigma(t+T_0)}}\right)\\
&+& \sigma(t+T_0)\alpha^2\rho^2\left({\frac{4\epsilon_t}{\sigma(t+T_0)} +\frac{4L^2}{\sigma^2(t+T_0)^2}}\right) \\
&\leq &  (1-\alpha)\left({\epsilon_t + \sqrt{\frac{4L^2\epsilon_t}{\sigma(t+T_0)}} + \frac{2L^2}{\sigma(t+T_0)}}\right) \\
&+& 4\alpha^2\rho^2\left({\epsilon_t +\frac{L^2}{\sigma(t+T_0)}}\right)\\ & \leq & 
\left({\epsilon_t + \sqrt{\frac{4L^2\epsilon_t}{\sigma(t+T_0)}} + \frac{2L^2}{\sigma(t+T_0)}}\right)\left({1-\alpha + 4\alpha^2\rho^2}\right),
\end{eqnarray*}
where the second inequality follows from Eq. \eqref{ie:sc_regretlemma:4}, the value of $r_t$, and using $(a+b)^2 \leq 2a^2+2b^2$ to upper bound $r_t^2$. The rest of the inequalities follows from simple algebraic manipulations.

Setting $\alpha = \frac{1}{5\rho^2}$ we have that
\begin{eqnarray*}
F_{t}(x_{t+1}) - F_{t}(x_{t+1}^*) \leq 
\left({\epsilon_t + \sqrt{\frac{4L^2\epsilon_t}{\sigma(t+T_0)}} + \frac{2L^2}{\sigma(t+T_0)}}\right)\left({1-\frac{1}{25\rho^2}}\right) .
\end{eqnarray*}

Plugging in our choice $\epsilon_t = \frac{(60\rho^2L)^2}{\sigma(t+T_0)}$ we have that
\begin{eqnarray*}
F_{t}(x_{t+1}) - F_{t}(x_{t+1}^*) &\leq & 
\frac{(60\rho^2L)^2}{\sigma(t+T_0)}\left({1 + \frac{1}{30\rho^2} + \frac{1}{1800(\rho^2)^2}}\right)\left({1-\frac{1}{25\rho^2}}\right) \\
&< & \frac{(60\rho^2L)^2}{\sigma(t+T_0)}\left({1 + \frac{1}{25\rho^2}}\right)\left({1-\frac{1}{25\rho^2}}\right) \\
&=& \frac{(60\rho^2L)^2}{\sigma(t+T_0)}\left({1-\frac{1}{(25\rho^2)^2}}\right)  .
\end{eqnarray*}
Finally, setting $T_0 = (25\rho^2)^2$ we have that
\begin{eqnarray*}
F_{t}(x_{t+1}) - F_{t}(x_{t+1}^*) &\leq &  \frac{(60\rho^2L)^2}{\sigma(t+T_0)}\left({1-\frac{1}{T_0}}\right)
< \frac{(60\rho^2L)^2}{\sigma(t+T_0)}\left({1-\frac{1}{t+1 + T_0}}\right) \\
&=& \frac{(60\rho^2L)^2}{\sigma(t+T_0)}\cdot\frac{t+T_0}{t + 1+ T_0} = \frac{(60\rho^2L)^2}{\sigma(t+1+T_0)} 
= \epsilon_{t+1} .
\end{eqnarray*}
\end{proof}

We are now ready to prove Theorem \ref{thm:sc}.
\begin{proof}
The proof follows the lines of the proof for  Theorem \ref{thr:mainthr}. Define the sequence of functions $\lbrace{\tilde{f}_t(x)}\rbrace_{t=0}^T$ in the following way:
\begin{eqnarray*}
\tilde{f}_0(x) = T_0\frac{\sigma}{2}\Vert{x-x_1}\Vert^2; \qquad \tilde{f}_t(x) = \nabla{}f_t(x_t) \cdot x + \frac{\sigma}{2}\Vert{x-x_t}\Vert^2 \quad \forall t\in\{1,2,...,T\}.
\end{eqnarray*}

Note that for all $t\in\{0,1,...,T\}$ it holds that $F_t(x) = \sum_{\tau=0}^t\tilde{f}_t(x)$, where $F_t(x)$ is as defined in Algorithm \ref{alg:rftlfw} for the case $\sigma > 0$, and recall that we define $F_0(x) := T_0\frac{\sigma}{2}\Vert{x-x_1}\Vert^2$. Recall that we define a sequence of points $\lbrace{x^*_t}\rbrace_{t=1}^{T+1}$ such that for all $t\in[T+1]$, $x^*_t = \arg\min_{x\in\mP}F_{t-1}(x)$. Let $x^* = \arg\min_{x\in\mP}\sum_{t=1}^Tf_t(x)$.

According to Lemma \ref{lem:be_the_leader} it holds that
\begin{eqnarray*}
\sum_{t=0}^T\tilde{f}_t(x_{t+1}^*) - \tilde{f}_t(x^*) &=& \sum_{t=1}^T\tilde{f}_t(x_{t+1}^*) - \tilde{f}_t(x^*) + T_0\frac{\sigma}{2}\left({\Vert{x_1^* - x_1}\Vert^2 - \Vert{x^* - x_1}\Vert^2}\right) \leq 0 .
\end{eqnarray*}

Using the upper bound $\Vert{x^*-x_1}\Vert \leq D$ and plugging the value of $T_0$ in Algorithm \ref{alg:rftlfw} we have that
\begin{eqnarray}\label{eq:8}
\sum_{t=1}^T\tilde{f}_t(x_{t+1}^*) - \tilde{f}_t(x^*) \leq \frac{T_0\sigma{}D^2}{2} = O(\sigma{}D^2\rho^4) .
\end{eqnarray}

Fix $t\in[T]$. It holds that
\begin{eqnarray*}
F_t(x_t^*) - F_t(x_{t+1}^*) &=& F_{t-1}(x_t^*) - F_{t-1}(x_{t+1}^*) + \tilde{f}_t(x_t^*) - \tilde{f}_t(x_{t+1}^*) \\
& \leq & \tilde{f}_t(x_t^*) - \tilde{f}_t(x_{t+1}^*) \leq L\Vert{x_t^* - x_{t+1}^*}\Vert,
\end{eqnarray*}
where the first inequality follows from the optimality of $x_t^*$ with respect to $F_{t-1}(x)$, and the second inequality follows from Lemma \ref{lemma:liphs} and using $L=G+\sigma{}D$.
Since $F_t(x)$ is $\sigma(t+T_0)$-strongly convex, this implies via Eq. \eqref{eq:strongconvexdist} that
\begin{eqnarray*}
\Vert{x_t^* - x_{t+1}^*}\Vert \leq \frac{2L}{\sigma(t+T_0)} .
\end{eqnarray*}

Applying Lemma \ref{sc_regretlemma} with the triangle inequality we have that
\begin{eqnarray*}
\Vert{x_t - x_{t+1}^*}\Vert \leq \frac{2L}{\sigma(t+T_0)} + \frac{60\rho^2L}{\sigma\sqrt{(t+T_0)(t-1+T_0)}} = O\left({\frac{\rho^2L}{\sigma{}t}}\right),
\end{eqnarray*}
where the equality follows since $T_0 \geq 1$ and by definition $\rho \geq 1$.

Thus, using Lemma \ref{lemma:liphs} again we have that
\begin{eqnarray*}
\tilde{f}_t(x_t) - \tilde{f}_t(x_{t+1}^*) =  O\left({\frac{\rho^2L^2}{\sigma{}t}}\right) .
\end{eqnarray*}

Plugging the above for all $t\in[T]$ into Eq. \eqref{eq:8} we have that
\begin{eqnarray*}
\sum_{t=1}^T\tilde{f}_t(x_t) - \tilde{f}_t(x^*) = O(\sigma{}D^2\rho^4)  + \sum_{t=1}^TO\left({\frac{\rho^2L^2}{\sigma{}t}}\right)
=  O\left({\sigma{}D^2\rho^4 + \frac{\rho^2L^2}{\sigma}\log{T}}\right) .
\end{eqnarray*}

The theorem now follows from the observation that since for all $t\in[T]$, $f_t(x)$ is $\sigma$-strongly convex it holds that
\begin{eqnarray*}
f_t(x_t) - f_t(x^*) &\leq & \nabla{}f_t(x_t) \cdot (x_t-x^*) - \frac{\sigma}{2}\Vert{x_t - x^*}\Vert^2 \\
&=& \nabla{}f_t(x_t)\cdot (x_t-x^*) - \frac{\sigma}{2}(\Vert{x_t - x^*}\Vert^2 - \Vert{x_t - x_t}\Vert^2) \\
&=& \tilde{f_t}(x_t) - \tilde{f_t}(x^*)
\end{eqnarray*}

\end{proof}

\subsection{Bandit Algorithm}
In this section we give an online algorithm for the \textit{partial information} setting (bandits). The derivation is basically straightforward using our algorithm for the \textit{full information} setting (Algorithm \ref{alg:rftlfw}) and the technique of \cite{FKM05}.

For this section we assume that the feasible set $\mP$ (again a polytope) is a full dimensional. We assume without loss of generality that the origin lies in the interior of $\mP$ (note that our Algorithms and the complexity measure $\mu$ are invariant to translation) and we denote by $r_0$ the largest scalar such that $\ball_{r_0}(0)\subset\mP$.

We assume that the loss function $f_t(x)$ chosen by the adversary on time $t$ is chosen with knowledge of the history of the game but without any knowledge of possible randomization used by the decision maker on time $t$ to produce his prediction. We further assume without loss of generality that for each function $f_t(x)$ it holds that $f_t(0) = 0$.

Note that since we assume that the gradients of each $f_t(x)$ are bounded in magnitude by $G$, it holds that $f_t(x)$ is $G$-Lipschitz. This follows since, 
\begin{eqnarray}\label{eq:banditLip}
\forall{x,y\in\mP}: \quad f_t(x)-f_t(y) \leq (x-y)\cdot\nabla{}f_t(x) \leq G\Vert{x-y}\Vert .
\end{eqnarray}

Also, since $f_t(0)=0$ and $0\in\mP$ we have that
\begin{eqnarray}\label{eq:banditbound}
\forall{x\in\mP}: \quad f_t(x) = f_t(x) - f_t(0) \leq (x-0)\cdot\nabla{}f_t(x) \leq GD .
\end{eqnarray}

The algorithm is given below. The algorithm uses our full-information algorithm - Algorithm \ref{alg:rftlfw} as a black box in the following way: on each round of the game, the bandit algorithm asks Algorithm \ref{alg:rftlfw} for his prediction for time $t$ - $x_t$, and then it generates a loss function $\tilde{f}_t(x)$ based on the received bandit feedback, and sends it to Algorithm \ref{alg:rftlfw} as the loss function for time $t$. In this way, the bandit algorithm ``simulates" a full-information game for Algorithm \ref{alg:rftlfw}.

\begin{algorithm}[H]
\caption{LLOO-based Bandit Optimization}
\label{alg:banditfw}
\begin{algorithmic}[1]
\STATE Input: time horizon $T$, upper bound on magnitude of gradients - $G$
\STATE Init Algorithm \ref{alg:rftlfw} with the LLOO from Section \ref{sec:llo} and input parameters $T,G, \sigma=0$
\STATE Set: $\gamma \gets \sqrt{\frac{nD}{r_0}}T^{-1/4}, \qquad \delta \gets r_0\gamma$
\FOR{$t=1...T$}
\STATE Receive point $x_t$ from Algorithm \ref{alg:rftlfw}
\STATE Sample a unit vector $u_t$ uniformly at random
\STATE Play $y_t = (1-\gamma)x_t + \delta{}u_t$
\STATE Receive $f_t(y_t)$
\STATE $g_t \gets \frac{n}{\delta}f_t(y_t)u_t$
\STATE Define the loss function $\tilde{f}_t(x) := g_t\cdot x$
\STATE Feed $\tilde{f}_t(x)$ to Algorithm \ref{alg:rftlfw}
\ENDFOR
\end{algorithmic}
\end{algorithm}

The regret analysis of Algorithm \ref{alg:banditfw} closely follows the analysis in \cite{FKM05}, but instead of using Zinkevich's algorithm \cite{Zinkevich03} for the reduction from bandit feedback to full feedback, we use our Algorithm \ref{alg:rftlfw}.

For a proof of the following Lemma see Lemma 2.1 in \cite{FKM05}.
\begin{lemma}\label{lem:banditGrad}
Fix $t\in[T]$ and define the function $\hat{f}_t(x) = \E_{v}[f_t(x+\delta{}v)]$, where $v$ is a vector sampled uniformly at random from the unit ball. Let $z_t = (1-\gamma)x_t$. Then, $\E_{u_t}[g_t] = \nabla{}\hat{f}_t(z_t)$.
\end{lemma}

In order to derive the regret bound for Algorithm \ref{alg:banditfw} we need the following technical lemma.
\begin{lemma}\label{lem:banditdist}
For all $t\in[T]$, let $\hat{f}_t(x)$ be as in Lemma \ref{lem:banditGrad}. It holds that 
\begin{eqnarray*}
\E\left[{\max_{x\in\mP}\sum_{t=1}^T(z_t-x)\cdot(\nabla\hat{f}_t(z_t)-g_t)}\right] \leq  \frac{\sqrt{T}nGD^2}{\delta}.
\end{eqnarray*}
\end{lemma}

\begin{proof}
It holds that
\begin{eqnarray*}
\E\left[{\max_{x\in\mP}\sum_{t=1}^T(z_t-x)\cdot(\nabla\hat{f}_t(z_t)-g_t)}\right]
= \E\left[{\sum_{t=1}^Tz_t\cdot(\nabla\hat{f}_t(z_t)-g_t)}\right]
+ \E\left[{\max_{x\in\mP}\sum_{t=1}^Tx\cdot(g_t-\nabla\hat{f}_t(z_t))}\right] .
\end{eqnarray*}

Note that by Lemma \ref{lem:banditGrad} we have that for all $t\in[T]$, $\E[z_t\cdot(\nabla\hat{f}_t(z_t)-g_t) \, | \, z_t] = 0$, and thus,
\begin{eqnarray*}
\E\left[{\max_{x\in\mP}\sum_{t=1}^T(z_t-x)\cdot(\nabla\hat{f}_t(z_t)-g_t)}\right]
= \E\left[{\max_{x\in\mP}\sum_{t=1}^Tx\cdot(g_t-\nabla\hat{f}_t(z_t))}\right] .
\end{eqnarray*}

Using the Cauchy-Schwartz inequality and the bound $\max_{x\in\mP}\Vert{x}\Vert \leq D$ we have that
\begin{eqnarray}\label{eq:10}
\E\left[{\max_{x\in\mP}x\cdot\left({\sum_{t=1}^Tg_t-\nabla\hat{f}_t(z_t)}\right)}\right] &\leq &
\E\left[{D \cdot \Vert{\sum_{t=1}^T\nabla\hat{f}_t(z_t)-g_t}\Vert}\right]  \nonumber \\
&=& D\cdot \E\left[{\Vert{\sum_{t=1}^T\nabla\hat{f}_t(z_t)-g_t}\Vert}\right].
\end{eqnarray}

It now holds that
\begin{eqnarray}\label{eq:11}
\E\left[{\Vert{\sum_{t=1}^T\nabla\hat{f}_t(z_t)-g_t}\Vert}\right]^2 &\leq& \E\left[{\Vert{\sum_{t=1}^T\nabla\hat{f}_t(z_t)-g_t}\Vert^2}\right] \nonumber \\
&=& \E\left[{\sum_{t=1}^T\Vert{\nabla\hat{f}_t(z_t)-g_t}\Vert^2 + 2\sum_{1 \leq i < j \leq T}(\nabla\hat{f}_i(z_i)-g_i)\cdot(\nabla\hat{f}_j(z_j)-g_j)}\right] \nonumber \\
&= & \sum_{t=1}^T\E\left[{\Vert{\nabla\hat{f}_t(z_t)-g_t}\Vert^2}\right] + 2\sum_{1 \leq i < j \leq T}\E\left[{(\nabla\hat{f}_i(z_i)-g_i)\cdot(\nabla\hat{f}_j(z_j)-g_j)}\right] \nonumber \\
&= & \sum_{t=1}^T\E\left[{\Vert{g_t - \E_{u_t}[g_t]}\Vert^2}\right] + 2\sum_{1 \leq i < j \leq T}\E\left[{(\nabla\hat{f}_i(z_i)-g_i)\cdot(\nabla\hat{f}_j(z_j)-g_j)}\right] \nonumber \\
&\leq & \sum_{t=1}^T\E\left[{\Vert{g_t}\Vert^2}\right] + 2\sum_{1 \leq i < j \leq T}\E\left[{(\nabla\hat{f}_i(z_i)-g_i)\cdot(\nabla\hat{f}_j(z_j)-g_j)}\right],
\end{eqnarray}
where the last equality follows from Lemma \ref{lem:banditGrad} and the last inequality follows since for any random vector $w$ it holds that $\E[\Vert{w-\E[w]}\Vert^2] \leq \E[\Vert{w}\Vert^2]$.

For all $j > i$ it holds that
\begin{eqnarray}\label{eq:12}
\E\left[{(\nabla\hat{f}_i(z_i)-g_i)\cdot(\nabla\hat{f}_j(z_j)-g_j)}\right]
&=& \E_{\lbrace{u_t}\rbrace_{t=1}^i}\left[{(\nabla\hat{f}_i(z_i)-g_i)\cdot\E_{\lbrace{u_{\tau}}\rbrace_{\tau=i+1}^j}\left[{(\nabla\hat{f}_j(z_j)-g_j) \, | \, \lbrace{u_t}\rbrace_{t=1}^i}\right]}\right] \nonumber \\
&=& 0,
\end{eqnarray}
where the last equality follows since according to Lemma \ref{lem:banditGrad}, the inner expectation is zero.

Plugging Eq. \eqref{eq:12} into Eq. \eqref{eq:11} for all $i<j$ we have that

\begin{eqnarray}\label{eq:13}
\E\left[{\Vert{\sum_{t=1}^T\nabla\hat{f}_t(z_t)-g_t}\Vert}\right]^2 \leq \sum_{t=1}^T\Vert{g_t}\Vert^2
\leq T\left({\frac{nGD}{\delta}}\right)^2,
\end{eqnarray}
where we have used Eq. \eqref{eq:banditbound} in the last inequality.

Plugging Eq. \eqref{eq:13} into Eq. \eqref{eq:10} we finally have that
\begin{eqnarray*}
\E\left[{\max_{x\in\mP}x\cdot\left({\sum_{t=1}^T\nabla\hat{f}_t(z_t)-g_t}\right)}\right]\leq  \frac{\sqrt{T}nGD^2}{\delta} .
\end{eqnarray*}

\end{proof}

\begin{theorem}
For $T \geq \left({\frac{Dn}{r_0}}\right)^2$ it holds that
the sequence of points $\lbrace{y_t}\rbrace_{t=1}^T$ produced by Algorithm \ref{alg:banditfw} is feasible and satisfies
\begin{eqnarray*}
\mathbb{E}\left[{\sum_{t=1}^Tf_t(y_t) - \min_{x^*\in\mP}\sum_{t=1}^Tf_t(x^*)}\right] =  O\left({GD\sqrt{\frac{nD}{r_0}}T^{3/4}  + GD\mu\sqrt{nT}}\right) ,
\end{eqnarray*}
where the expectation is with respect to the randomness in choosing the vectors $\lbrace{u_t}\rbrace_{t=1}^T$.
\end{theorem}
\begin{proof}
First, we prove the feasibility of the sequence of points $\lbrace{y_t}\rbrace_{t=1}^T$. Fix $t\in[T]$. By Algorithm \ref{alg:banditfw} we have that $y_t = (1-\gamma)x_t + \delta{}u_t$, where $x_t\in\mP$ and $u_t$ is a unit vector. Since by our assumption it holds that $\ball_{r_0}(0)\subset\mP$, we have that $r_0u_t\in\mP$. Thus, for any $\gamma\in[0,1]$ and $\delta \in[0,\gamma{}r_0]$, it follows that $y_t\in\mP$. Clearly, the values of $\delta,\gamma$ in Algorithm \ref{alg:banditfw} satisfy these requirements.

We move to prove the regret bound.

By applying Theorem \ref{thr:mainthr} with respect to the loss functions $\lbrace{\tilde{f}_t(x) = g_t \cdot x}\rbrace_{t=1}^T$ we have that
\begin{eqnarray*}
\sum_{t=1}^Tx_t\cdot g_t - \min_{x^*\in\mP}\sum_{t=1}^Tx^*\cdot g_t = O\left({GD\rho\sqrt{T}}\right),
\end{eqnarray*}

where we recall that $\rho = \sqrt{n}\mu$.

For $t\in[T]$, denote $z_t = (1-\gamma)x_t$. Since $(1-\gamma)\in[0,1]$, it holds that
\begin{eqnarray}\label{eq:9}
\sum_{t=1}^Tz_t\cdot g_t - (1-\gamma)\min_{x^*\in\mP}\sum_{t=1}^Tx^*\cdot g_t 
&=& \sum_{t=1}^Tz_t\cdot g_t - \min_{x^*\in(1-\gamma)\mP}\sum_{t=1}^Tx^*\cdot g_t  \nonumber \\
&=& O\left({GD\rho\sqrt{T}}\right) .
\end{eqnarray}

Ranging we have that

\begin{eqnarray*}
\max_{x^*\in(1-\gamma)\mP}\sum_{t=1}^T(z_t-x^*)\cdot g_t = O\left({GD\rho\sqrt{T}}\right) .
\end{eqnarray*}

Plugging in Lemma \ref{lem:banditdist} and taking expectation we have that

\begin{eqnarray*}
\E\left[{\max_{x^*\in(1-\gamma)\mP}\sum_{t=1}^T(z_t-x^*)\cdot \nabla\hat{f}_t(z_t)}\right] = O\left({\frac{\sqrt{T}nGD^2}{\delta} + GD\rho\sqrt{T}}\right) .
\end{eqnarray*}

Note that since $f_t(x)$ is a convex function, so is $\hat{f}_t(x)$ and thus we have that
\begin{eqnarray}\label{eq:15}
\E\left[{\sum_{t=1}^T\hat{f}_t(z_t) - \min_{x^*\in(1-\gamma)\mP}\sum_{t=1}^T\hat{f}_t(x^*)}\right]  =  O\left({\frac{\sqrt{T}nGD^2}{\delta}  + GD\rho\sqrt{T}}\right) .
\end{eqnarray}

Fix $t\in[T]$. Note that since $f_t(x)$ is $G$-Lipschitz (see Eq. \eqref{eq:banditLip}), for all $x\in\mP$ it holds that
\begin{eqnarray}\label{eq:14}
\vert{\hat{f}_t(x) - f_t(x)}\vert &=& \vert{\E_{v\in\ball}[f_t(x+\delta{}v) - f_t(x)]}\vert \nonumber
\\
&\leq &\E_{v\in\ball}[\vert{f_t(x+\delta{}v)-f_t(x)}\vert] \leq G\delta .
\end{eqnarray}

Plugging Eq. \eqref{eq:14} for all $t\in[T]$ into Eq. \eqref{eq:15} we have that
\begin{eqnarray*}
\E\left[{\sum_{t=1}^Tf_t(z_t) - \min_{x^*\in(1-\gamma)\mP}\sum_{t=1}^Tf_t(x^*)}\right] =  O\left({TG\delta+ \frac{\sqrt{T}nGD^2}{\delta} +  GD\rho\sqrt{T}}\right) .
\end{eqnarray*}

Since for all $t\in[T]$, $\Vert{y_t-z_t}\Vert \leq \delta$, using the Lipschits property of $f_t(x)$ we have that
\begin{eqnarray*}
\E\left[{\sum_{t=1}^Tf_t(y_t) - \min_{x^*\in(1-\gamma)\mP}\sum_{t=1}^Tf_t(x^*)}\right] =  O\left({TG\delta+ \frac{\sqrt{T}nGD^2}{\delta} +  GD\rho\sqrt{T}}\right) .
\end{eqnarray*}

Using again Eq. \eqref{eq:banditLip}, \eqref{eq:banditbound} and the fact that $0\in\mP$, $f_t(0)=0$, we have that for all $x\in\mP$ and $t\in[T]$, $f_t((1-\gamma)x) \leq (1-\gamma)f_t(x) \leq f_t(x) + \gamma{}GD$. Thus we have that
\begin{eqnarray*}
\E[\regret_T] &:=& \E\left[{\sum_{t=1}^Tf_t(y_t) - \min_{x^*\in\mP}\sum_{t=1}^Tf_t(x^*)}\right] \\
& =& O\left({T\gamma{}GD + TG\delta + \frac{\sqrt{T}nGD^2}{\delta}  + GD\rho\sqrt{T}}\right) .
\end{eqnarray*}

Now, setting $\delta = \gamma{}r_0$ as in Algorithm \ref{alg:banditfw}, and recalling that $D \geq r_0$, we have that
\begin{eqnarray*}
\E[\regret_T] = O\left({T\gamma{}GD + \frac{\sqrt{T}nGD^2}{\gamma{}r_0}  + GD\rho\sqrt{T}}\right) .
\end{eqnarray*}

Finally, setting $\gamma = \sqrt{\frac{Dn}{r_0}}T^{-1/4}$ we have that
\begin{eqnarray*}
\E[\regret_T] = O\left({GD\sqrt{\frac{nD}{r_0}}T^{3/4}  + GD\rho\sqrt{T}}\right) .
\end{eqnarray*}

\end{proof}

\section{Lower bound} \label{sec:lowerbounds}

In this section we revisit our main result from Section \ref{sec:offline_opt}, that is, our linearly converging algorithm for smooth and strongly convex optimization over polytopes. We show that in certain settings our convergence rate (i.e., number of calls to the linear optimization oracle to reach a certain approximation error) is in fact nearly tight and cannot be improved beyond constants and logarithmic terms for conditional gradient-like algorithms, i.e., algorithms that can request a vertex of the polytope that minimizes the dot product with a certain linear objective and take linear combinations of these vertices. Similar arguments appear in \cite{Jaggi13b, Lan13}.

Towards this end, consider the following optimization problem:
\begin{eqnarray}\label{eq:lowboundprob}
\min_{x\in\mS_n}\lbrace{f(x) := \frac{1}{2}\Vert{x-\frac{1}{n}\vec{1}}\Vert^2}\rbrace,
\end{eqnarray}
where $\mS_n$ is the probabilistic simplex in $\mathbb{R}^n$ and $\vec{1}$ is the all-ones vector in $\mathbb{R}^n$. Note that the feasible point $x^* = \frac{1}{n}\vec{1}$ is the optimal solution to Problem \eqref{eq:lowboundprob} with value $f(x^*) = 0$. Note also that $f(x)$ is $1$-smooth and $1$-strongly convex.

Given a conditional gradient-like algorithm, as described above, and assuming without losing generality that the initial iterate of the algorithm $x_1$ is a vertex of $\mS_n$, let $x_t$ denote the last feasible iterate that the algorithm produced by making no more than $t-1$ queries to the linear optimization oracle of $\mS_n$. Since, as we assume, each call to the linear optimization oracle returns a vertex, i.e., a member of the standard basis in $\mathbb{R}^n$, it follows that $\Vert{x_t}\Vert_0 \leq t$ (here $\Vert{\cdot}\Vert_0$ measures the number of non-zeros). It follows from a simple calculation that
\begin{eqnarray*}
f(x_t) - f(x^*)= \frac{1}{2}\Vert{x_t - \frac{1}{n}\vec{1}}\Vert^2 \geq \frac{1}{2}(n-t)\cdot\frac{1}{n^2} \geq \frac{1}{4n} \quad \forall t\leq \frac{n}{2} .
\end{eqnarray*}

It thus follows that in order for a conditional gradient-like method to solve Problem \eqref{eq:lowboundprob} up to an error of at most $\frac{1}{4n}$, it requires $\Omega(n)$ calls to the linear optimization oracle of $\mS_n$.

We now show that for Problem \eqref{eq:lowboundprob} our algorithm, Algorithm \ref{alg:offline}, nearly matches the above lower bound. To see this, note that we can write $\mS_n$ in the following way: $\mS_n = \lbrace{x\in\mathbb{R}^n \, | \, -Ix \leq \vec{0}, \, x\cdot\vec{1} = 1}\rbrace$. It thus follows that for the simplex we have that $\psi = 1$,$\xi = 1$,$D=\sqrt{2}$ which implies that $\mu = O(1)$. Plugging this into Theorem \ref{thm:offthm} we have that Algorithm \ref{alg:offline} requires $O(n\log(1/\epsilon)$ calls to the linear optimization oracle in order to produce an $\epsilon$-approximated solution to Problem \eqref{eq:lowboundprob} for any $\epsilon > 0$, which matches the above lower bound up to a logarithmic factor.

\section*{Acknowledgments.}
% Enter the text of acknowledgments here

The authors would like to thank Arkadi Nemirovski for numerous helpful comments on an earlier draft of this paper. 

%This work was supported by the European Research Council (project SUBLRN).

\bibliographystyle{plain}
\bibliography{bib}

\appendix

\section{Proof of Lemma \ref{lem:be_the_leader}}

For clarity, we first restate the lemma and then prove it.
\begin{lemma}
Let $\lbrace{f_t(x)}\rbrace_{t=1}^T$ be a sequence of loss functions and let $\lbrace{x^*_t}\rbrace_{t=1}^T$ be a sequence of points such that for all $t\in[T]$, $x^*_t \in\arg\min_{x\in\mP}\sum_{\tau=1}^tf_{\tau}(x)$. Then it holds that
\begin{eqnarray*}
\sum_{t=1}^Tf_t(x^*_t) - \min_{x\in\mP}\sum_{t=1}^Tf_t(x) \leq 0 .
\end{eqnarray*}
\end{lemma}
\begin{proof}
We prove by induction that for any $\tau\in[T]$ it holds that
\begin{eqnarray*}
\sum_{t=1}^{\tau}f_t(x^*_t) - \min_{x\in\mP}\sum_{t=1}^{\tau}f_t(x) \leq 0 .
\end{eqnarray*}

For the base case $\tau = 1$ the claim clearly holds since $x^*_1 \in\arg\min_{x\in\mP}f_1(x)$. Assume now that the claim holds for some $\tau \geq 1$. On time $\tau+1$ it holds that
\begin{eqnarray*}
\sum_{t=1}^{\tau+1}f_t(x^*_t) - \min_{x\in\mP}\sum_{t=1}^{\tau+1}f_t(x)
&\leq & \min_{y\in\mP}\sum_{t=1}^{\tau}f_t(y) + f_{\tau+1}(x^*_{\tau+1})  - \min_{x\in\mP}\sum_{t=1}^{\tau+1}f_t(x) \\
&\leq & \sum_{t=1}^{\tau}f_t(x^*_{\tau+1}) + f_{\tau+1}(x^*_{\tau+1})  - \min_{x\in\mP}\sum_{t=1}^{\tau+1}f_t(x) \\
&=& \sum_{t=1}^{\tau+1}f_t(x^*_{\tau+1}) - \min_{x\in\mP}\sum_{t=1}^{\tau+1}f_t(x) 
 = 0 ,
\end{eqnarray*}
where the first inequality follows from the induction hypothesis and the third one from the optimality of $x^*_{\tau+1}$.
\end{proof}

\end{document}